\documentclass{article}

\usepackage[preprint]{neurips_2026}

\usepackage[utf8]{inputenc}
\usepackage[T1]{fontenc}
\usepackage{mathrsfs}

\usepackage{hyperref}
\usepackage{url}
\usepackage{booktabs}
\usepackage{amsfonts}
\usepackage{nicefrac}
\usepackage{microtype}
\usepackage{xcolor}

\usepackage{amsmath,amssymb,amsthm}
\usepackage{bm}

\usepackage{graphicx}
\usepackage{subcaption}
\usepackage{float}
\usepackage{placeins}
\usepackage{enumitem}

\usepackage{algorithm}
\usepackage{algpseudocode}

\newtheorem{theorem}{Theorem}

\newtheorem{proposition}{Proposition}

\allowdisplaybreaks
\DeclareMathOperator*{\argmax}{arg\,max}

\title{Variational  Smoothing and Inference for SDEs from Sparse Data with Dynamic Neural Flows}

\author{
Yu Wang\thanks{Equal contribution.} \\
Department of Mathematics\\
Louisiana State University\\
\texttt{yuwang@lsu.edu}
\And
Arnab Ganguly\footnotemark[1] \thanks{Corresponding author.} \\
Department of Mathematics\\
Louisiana State University\\
\texttt{aganguly@lsu.edu}
}

\begin{document}

\maketitle
\begin{abstract}

Stochastic differential equations (SDEs) provide a flexible framework for modeling temporal dynamics in partially observed systems. A central task is to calibrate such models from data, which requires inferring latent trajectories and parameters from sparse, noisy observations. Classical smoothing methods for this problem are often limited by path degeneracy and poor scalability.
In this work, we developed a novel method based on characterization of the posterior SDE in terms of conditional backward-in-time score defined as the gradient of a function solving a Kolmogorov backward equation with multiplicative updates at observation times. We learn this conditional score using  neural networks trained to satisfy both the governing PDE and the observation-induced jump conditions, thereby integrating continuous-time dynamics with discrete Bayesian updates. The resulting score induces a posterior SDE with the same diffusion coefficient but a modified drift, enabling efficient posterior trajectory sampling.
We further derive a likelihood-based objective for learning the SDE parameters, yielding an evidence lower bound (ELBO) for joint state smoothing and parameter estimation. This leads to a variational EM-style procedure, where the neural conditional score is optimized to approximate the smoothing distribution, followed by a maximization step over the SDE parameters using samples from the induced posterior.
Experiments on nonlinear systems demonstrate accurate and stable inference with a very few observations demonstrating significant improved scalability compared to classical MCMC methods.

% To bypass intractable high-dimensional Kolmogorov equations, we approximate this filter using Physics-Informed Neural Networks (PINNs), seamlessly integrating continuous-time PDE constraints with discrete Bayesian observation updates. 
\end{abstract}

\section{Introduction}
Stochastic differential equations (SDEs) provide a flexible framework for modeling dynamical systems evolving under uncertainty. They have also emerged as a cornerstone of generative modeling, where they serve to transform simple noise to complex data distributions \cite{TzRa19-1, TzRa19-2, SoSoKi21, CaTaGa24, MMYG23,MPS25}. In many applications, however, the latent state process is not directly observed; instead, one has access only to sparse and noisy measurements at discrete time points. A central problem is therefore to infer both the unknown parameters governing the dynamics and the distribution of latent trajectories conditioned on these observations, commonly referred to as the \emph{smoothing or posterior distribution}.

Parametric inference for SDEs has been extensively studied (e.g., see \cite{Yos92, Kes97, Chib01, RoSt01, Saha02, Kut03, GoWi05, BPRF06, FePaRo08, Saha08, GoWi08, Iacus08, ChCh11, ArOp11, CsOp13, Li13, BlSo14, WGBS17, Bis22a, GaMiZh25, GaMiZh23}). A typical approach to this problem involves  discretizing the latent or prior SDE in time and performing inference on the resulting high-dimensional latent state vector. Methods based on particle filtering, or Markov chain Monte Carlo (MCMC) are widely used \cite{Chib01, GoWi08, FePaRo08, GoWi11, WGBS17-2}, but often suffer from path degeneracy, poor mixing, and limited scalability, especially when observations are sparse or and the latent dynamics is highly nonlinear. 

In this work, we take a different perspective based on a path-space characterization of the smoothing distribution. We utilize the fact that the smoothing distribution over trajectories can be described as the law of a diffusion process, whose drift is the prior drift modified by a \emph{backward-in-time conditional score}. The latter is defined as the spatial gradient of the logarithm of a message function solving a backward Kolmogorov equation with multiplicative updates at observation times. This yields a continuous-time formulation of Bayesian conditioning at the level of path measures.

Building on this characterization, we develop a scalable method for joint smoothing and parameter estimation. The key idea is to approximate the backward conditional score by a family of neural networks trained to satisfy the governing partial differential equations (PDEs) together with the observation-induced jump conditions. This yields a tractable approximation of the posterior SDE, from which approximate trajectories can be efficiently sampled without resorting to high-dimensional discretizations of the prior process.

We next derive a likelihood-based objective for learning the SDE parameters in the form of an evidence lower bound (ELBO). This leads to a variational Expectation Maximization (EM)-type procedure: for a given iterate of the SDE parameter, the E-step consists of training the neural network to approximate the conditional score and  and using it to define a tractable posterior SDE whose law is the approximate smoothing distribution. Samples from this induced SDE are then used to compute a Monte Carlo approximation of the corresponding ELBO, thereby approximating the intractable conditional expectation required in the EM objective. In the M-step, the next iterate of SDE parameter is obtained by maximizing this Monte Carlo ELBO .

We summarize the main components of our approach as follows:

\begin{itemize}
    \item \textbf{Characterization of smoothing:}  We utilize a representation of the smoothing distribution as the law of a posterior SDE whose drift equals the prior drift plus a correction given by a backward-in-time conditional score.

    \item \textbf{Neural approximation of backward score:} We approximate this  conditional score using neural networks, yielding a tractable posterior SDE that can be efficiently sampled without resorting to MCMC.

    \item \textbf{Variational EM:} We construct a variational EM procedure based on a Monte Carlo approximation of the ELBO using reparameterized samples from the learned posterior SDE, for joint smoothing and parameter estimation.
\end{itemize}

The resulting framework integrates continuous-time stochastic dynamics, partial observation, and variational inference in a novel unified manner. In contrast to classical discretization-based approaches, our method operates directly at the level of the smoothing distribution, leading to improved stability and scalability, particularly in regimes with sparse observations. Empirical results on nonlinear systems demonstrate accurate trajectory reconstruction and parameter estimation, with significant gains over standard MCMC-based methods.

\section{Model and problem setup}

Let $(\Omega,\mathcal{F}, \{\mathcal{F}_t\},\mathbb{P})$ be a filtered probability space supporting a  $d$-dimensional Brownian motion $W$. Our model consists of the the latent SDE:
\begin{equation}\label{eq:signal}
dX(t) = \drft(\kappa, X(t))\,dt + \dffun(X(t))\,dW(t), 
\quad X(0)\sim \initdist,
\end{equation}
where $\kappa \in K_0 \subset \R^{d_0}$ is the  model parameter, and the parameter set $K_0$ is assumed to be compact. We assume that for each $\kappa$, the SDE admits a strong unique solution (e.g., under linear growth and local Lipschitz conditions on the coefficients). We denote a sample path until time $T$ by $X^\kappa_{[0,T]}$.

{\bf Data:} The latent process $X$ is not observed directly. Our data comprises noisy, partial measurements $\by_{1:M_0} = (y_1,\dots,y_{M_0})$ at times $\{t_m\}_{m=1}^{M_0} \subset [0,T]$ --- a realization of $\bm{\obsY}_{t_1:t_{M_0}} = (\obsY(t_1), \obsY(t_2), \hdots, \obsY(t_{M_0}))$. We assume, without loss of generality, $t_{M_0} =T$. For each observation time $t_m$, $\obsY(t_m)$ is an $\R^{d_0}$-valued random variable with $d_0 \leq d$ satisfying 
\begin{align*}
\obsY(t_m) | X(t_m) = x_m \ \sim \ \obsden(\cdot| x_m) .
\end{align*}
A widely used observation model is the linear Gaussian  model:
\begin{align}\label{eq:obs}
	\obsY(t_m) = \omat X(t_m) +\vep_m, \quad \text{ with } \vep_m \stackrel{iid}\sim \No_{d_0}(0, \Sigma_{\mrm{noise}}),
\end{align}
where $G \in \R^{d_0\times d}$ and the covariance matrix of the noise, $\Sigma_{\mrm{noise}} $, is p.d. In this case,  the conditional observation density $\obsden$ is given by $\obsden(\cdot| \procX(t_m) =x_m) = \No_{d_0} (\cdot|\omat x_m, \Sigma_{\mrm{noise}})$.

{\bf Goal:}
Given $\by_{1:M_0}$, our objectives are:
(i) to approximate the smoothing (or posterior) distribution 
$\Post^{(\kappa)}(\cdot \mid \by_{1:M_0})$ defined as the conditional distribution of full latent trajectory $X_{[0,T]}$ given $ \bm{\obsY}_{t_1:t_{M_0}} = \by_{1:M_0}$, and 
(ii) to estimate $\kappa$. 

We emphasize that $\Post^{(\kappa)} (\cdot|\by_{1:M_0})$ and $\Pre^{(\kappa)} \dfeq \mrm{Law}(X_{[0,T]})$ are probability measures on the path space $C([0,T],\R^d)$, rather than a collection of marginal distributions at discrete times; specifically,  for a measurable set $\cal{A} \subset C([0,T],\R^d),$
$$\PP(X_{[0,T]} \in \cal{A}) = \Pre^{(\kappa)}(\cal{A}), \quad \PP(X_{[0,T]} \in \cal{A} \mid \bm{\obsY}_{t_1:t_{M_0}} = \by_{1:M_0}) = \Post^{(\kappa)} (\cal{A}|\by_{1:M_0}).$$
Moreover, $\Post^{(\kappa)}(\cdot \mid \by_{1:M_0})$ is absolutely continuous with respect to $\Pre^{(\kappa)}$ satisfying
\begin{align}
    \label{eq:prior-post-RN}
    \f{d\Post^{(\kappa)} (\cdot \mid \by_{1:M_0}) }{d\Pre^{(\kappa)}} (X_{[0,T]}) = \f{\prod_{m=1}^{M_0} \obsden(y_m \mid X(t_m))}{ e^{\ell(\kappa \mid \by_{1:M_0})}},
\end{align}
where 
\begin{align}
    \label{eq:log-like}
  \ell(\kappa \mid \by_{1:M_0}) = \ln \EE_\kappa\!\left[\prod_{m=1}^{M_0} \obsden(y_m \mid X(t_m))\right].  
\end{align}
is the log-likelihood function of $\kappa$ given the observations. Thus the maximum likelihood estimate (MLE) of $\kappa$ is given by $\hat \kappa_{\mrm{mle}} = \argmax _{\kappa} \ell(\kappa \mid \by_{1:M_0})$.

\section{Methodology}
Estimation of the SDE parameter $\kappa$ and the smoothing distribution 
$\Post^{(\kappa)}(\cdot \mid \by_{1:M_0})$ are intricately tied. The starting point of our approach is the following variational characterization of the log-likelihood. We provide the proof in Appendix~B \eqref{eq:kl-sde}

\begin{theorem}[Variational representation of the log-likelihood]
\label{thm:elbo-lower-bound}
For any $\kappa \in K_0$,
\begin{equation}\label{eq:like-elbo}
\ell(\kappa \mid \by_{1:M_0})
=
\sup_{Q \in \mathcal{P}(C([0,T], \mathbb{R}^d))}
\mathrm{ELBO}(Q, \kappa \mid \by_{1:M_0}),
\end{equation}
where
\begin{align}\label{eq:elbo-clean}
\mathrm{ELBO}(Q, \kappa \mid \by_{1:M_0})
=
\mathbb{E}_Q \!\left[
\sum_{m=1}^{M_0} \ln \obsden(y_m \mid X(t_m))
\right]
-
\mathrm{KL}\big(Q \,\|\, \Pre^{(\kappa)}\big).
\end{align}
Moreover, the supremum is attained at 
$
Q^* = \Post^{(\kappa)}(\cdot \mid \by_{1:M_0}).
$
\end{theorem}
Consequently,
$
\hat{\kappa}_{\mathrm{MLE}}
=
\arg\max_\kappa \sup_Q \mathrm{ELBO}(Q,\kappa),
$
which leads to a variational EM-type procedure consisting of alternating maximization over $Q$ and $\kappa$:
\begin{itemize}%[leftmargin=*]
\item (E-step) Fix $\kappa_i$, and construct an approximation 
$\widetilde Q_i(\cdot \mid \by_{1:M_0})$ to $\Post^{(\kappa_i)}(\cdot \mid \by_{1:M_0})$.\\
Compute ELBO.
\item (M-step) Update
$\dst 
\kappa_{i+1}
=
\arg\max_\kappa \mathrm{ELBO}(\widetilde Q_i, \kappa \mid \by_{1:M_0}).$
% \item (M-step) Update
% $\dst 
% \kappa_{i+1}
% =
% \arg\max_\kappa \mathrm{ELBO}(\widetilde Q_i, \kappa \mid \by_{1:M_0}).
% $
\end{itemize}
The KL term in \eqref{eq:elbo-clean} is finite only when $Q$ is the law of an SDE with the same diffusion coefficient as the prior. In this case, it admits an explicit representation via Girsanov’s theorem (see \eqref{eq:kl-sde}).
This plays a crucial role in our learning procedure. 

\subsection{E-Step (first half): Approximation of smoothing distribution}
{\bf Variational characterization of posterior on path space:} We now describe a method for approximating the smoothing distribution $\Post^{(\kappa)}(\cdot \mid \by_{1:M_0})$ for a given $\kappa$.  
A key theoretical ingredient underlying our method  is the path-space characterization of the smoothing distribution. Specifically, the posterior measure $\Post^{(\kappa)}(\cdot \mid \by_{1:M_0})$ can be characterized as the law of a diffusion process $\bar X$ solving the \emph{posterior SDE}
\begin{align}\label{eq:SDE-post}
d\bar X(t)
= \bar b(\kappa, t, \bar X(t))\,dt
+ \sigma(\bar X(t))\, dW(t),
\qquad a(x) = \sigma(x)\sigma(x)^\top,
\end{align}
where the drift is given by
\begin{align}\label{eq:post-drift}
\bar b(\kappa, t, x)
= b(\kappa, x) + a(x)\nabla_x \ln w(\kappa, x, t).
\end{align}
Here, the function $w(\kappa,x,t)$ represents the conditional likelihood of the future observations $\by_{m+1:M_0}$ given the current state $X(t)=x$, for $t \in (t_m, t_{m+1}]$, and admits the representation
\begin{equation}\label{eq:hrepr}
w(\kappa, x, t)
= p^{(\kappa)}\big(\by_{m+1:M_0} \mid X(t)=x\big)
= \mathbb{E}_\kappa\!\left[
\prod_{k>m} \obsden\big(y_k \mid X(t_k)\big)
\,\middle|\, X(t)=x
\right].
\end{equation}
By \eqref{eq:kol-back}, the function $w$ satisfies a backward Kolmogorov equation with multiplicative updates at observation times. Specifically, for $t \in (t_m,t_{m+1}]$,  $w$ evolves backward in time according to
\begin{equation*}%\label{eq:backwardPDE-post}
\partial_t w(\kappa, x,t) + \cal{A}^{(\kappa)} w(\kappa, x,t) = 0,
\end{equation*}
where the infinitesimal generator $\cal{A}^{(\kappa)}$ of the SDE \eqref{eq:signal} is 
given by
\begin{align*}%\label{eq:gen-X}
	\mathcal{A}^{(\kappa)} f(x) = \sum_{i=1}^d b_i(x; \kappa) \frac{\partial f(x)}{\partial x_i} + \frac12 \sum_{i,j=1}^d a_{ij}(x) \frac{\partial^2 f(x)}{\partial x_i \partial x_j}, \quad f \in C^2(\R^d, \R),
\end{align*}
At each observation time $t_m$, $w$ undergoes a multiplicative update:
\begin{equation}\label{eq:jump-obs}
w(\kappa, x,t_m^-) = \obsden(y_m\mid x)\,w(\kappa, x,t_m^+),
\end{equation}
with terminal condition $w(\kappa,x,T)=\obsden(y_{M_0}\mid x)$.

Since $w$ is not available in closed form, \eqref{eq:SDE-post} and \eqref{eq:post-drift} suggest approximating the smoothing distribution $\Post^{(\kappa)}(\cdot \mid \by_{1:M_0})$ by the law $\widetilde Q$ of a controlled diffusion $\widetilde X$ on $[0,T]$ of the form
\begin{align}\label{eq:SDE-post-approx}
d\widetilde X(t)
=
\widetilde b(\kappa,t,\widetilde X(t))\,dt
+
\sigma(\widetilde X(t))\,dW(t),
\qquad
\widetilde b(\kappa,t,x)
=
b(\kappa,x) + a(x)\widetilde s(\kappa,x,t),
\end{align}

Related ideas have been explored in \cite{SuGa16} and \cite{OpArch09, ArOp11}, where the posterior SDE is approximated by restricting its time marginals to lie within a parametric family (exponential family in \cite{SuGa16} and Gaussian in \cite{OpArch09, ArOp11}). However, constraining only the time marginals does not in general yield an accurate approximation of the smoothing distribution over path space, and further restricting these marginals to simple parametric families significantly limits their expressiveness.

{\bf Approximation of the conditional score:}
In contrast, we obtain $\tilde s(\kappa,\cdot, \cdot)$ by direct approximation of the {\em backward  conditional score}
\[
s(\kappa,x,t) \doteq \nabla_x \ln w(\kappa,x,t),
\]
via a neural parameterization. This leads to an approximation of $\Post^{(\kappa)}$ in the path-space thereby avoiding simpler parametric assumptions only on time marginals.

Note that $w(\kappa,x,t)$ is not a probability density in $x$, but admits a Feynman-Kac representation. Consequently, $s(\kappa,x,t)$ is not a Fisher score in the classical sense, but simply the spatial gradient of the log of this backward message function.

Note that accurate approximation of the conditional score $s(\kappa,\cdot,\cdot)$ requires accuracy of approximation of the backward message function $w(\kappa, \cdot, \cdot)$ in log-scale.  We thus directly approximate its logarithm: $h(\kappa, x, t) \dfeq \ln w(\kappa, x, t)$. The following result derives the explicit PDE for $h(\kappa, \cdot,\cdot)$. Refer to Appendix~B \eqref{eq:kl-sde} for the proof.

\begin{proposition}
    \label{prop: log-w-pde}
 For any $\kappa \in K_0$, $h(\kappa, \cdot, \cdot)$ satisfies the backward PDE between observation times $(t_{m-1}, t_{m}]$ 
 \begin{equation} \label{eq:backwardPDE-log-post}
\partial_t h(\kappa, x, t) 
+ \mathcal{A}^{(\kappa)} h(\kappa, x, t) 
+ \frac{1}{2} \nabla_x^\top h(\kappa, x, t)\, a(x)\, \nabla_x h(\kappa, x, t) 
= 0,
\end{equation}
with the boundary condition at $t_m$ given by 
\begin{equation} \label{eq:jump-obs-lnw}
h(\kappa, x, t_m) 
= \ln \rho_{\mathrm{obs}}(y_m \mid x) + h(\kappa, x, t_m^+).
\end{equation}
\end{proposition}
{\em Key idea:} Since $h$ evolves differently across intervals $(t_{m-1}, t_m]$, to approximate $h$ for a fixed $\kappa$ we employ a collection of neural networks (NNs) $\{\widetilde h_{\theta_m}\}_{m=1}^{M_0}$ with parameters $\bm{\theta} = (\theta_1,\dots,\theta_{M_0})$.  To train these networks, we enforce the PDE \eqref{eq:backwardPDE-log-post} over each interval $(t_{m-1},t_m]$ together with the jump condition \eqref{eq:jump-obs-lnw} at $t_m$. For this purpose, we introduce a reference probability measure $p_{\mathrm{ref}}$ on $\mathbb{R}^d$ (e.g. a Gaussian distribution or Uniform distribution over some compact set) and sample space--time points $(x,t)$ with $x \sim p_{\mathrm{ref}}$ and $t \sim U(t_{m-1},t_m]$. The networks $\{\widetilde h_{\theta_m}\}$ are then trained {\em simultaneously} so that these conditions are approximately satisfied at the sampled points $(x,t)$.
This leads to the following loss function for determining $\bm{\theta} = (\theta_1,\dots,\theta_{M_0})$:
\begin{align*}
\mathcal{L}_{\mathrm{smooth}}(\kappa, \bm{\theta})
&= \sum_{m=1}^{M_0}\Big[\EE_{\substack{x \sim p_{\mathrm{ref}} \\ t \sim U(t_{m-1},t_m]}}
\Big\|
\partial_t \widetilde h_{\theta_m}(\kappa, x,t) 
+ \mathcal{A}^{(\kappa)} \widetilde h_{\theta_m}(\kappa, x,t) 
 + \tfrac{1}{2} \nabla_x^\top \widetilde h_{\theta_m} a(x) \nabla_x \widetilde h_{\theta_m}
\Big\|^2  \\
& \quad + 
\EE_{x \sim p_{\mrm{ref}}}\Big\|
\widetilde h_{\theta_m}(\kappa, x, t_m) 
- \ln \rho_{\mathrm{obs}}(y_m \mid x) 
- \widetilde h_{\theta_{m+1}}(\kappa, x, t_m^+)
\Big\|^2\Big].
\end{align*}
The optimal $\theta$ is then given by $\bm{\theta}^* \equiv \bm{\theta}^*(\kappa)  = \arg\min_{\bm{\theta}} \mathcal{L}_{\mathrm{smooth}}(\kappa, \bm{\theta})$, and the learned approximate conditional score function  by
\begin{align}\label{eq:learned-score}
\widetilde s^*(\kappa,x,t) \equiv \widetilde s_{\bm{\theta}^* }(\kappa,x,t) = 
\nabla_x \widetilde h_{\theta_m^*}(\kappa,x,t),
\quad t \in (t_{m-1},t_m].
\end{align}

{\bf Remark:} Importantly, if the observation density is known up to a multiplicative constant independent of $x$, i.e.,
$$
\obsden(y_m \mid x) = C_m \, q_{\mathrm{obs}}(y_m \mid x),
\quad C_m>0,
$$
then $\ln \obsden(y_m \mid x)$ differs from $\ln q_{\mathrm{obs}}(y_m \mid x)$ by an additive constant $\ln C_m$, which does not affect the PDE residuals or jump constraints in the training. Specifically, if $\breve h_{\theta_m}$ denotes the NN trained with $\ln q_{\mathrm{obs}}(y_m \mid x)$ in place of $\ln \obsden(y_m \mid x)$ in $\mathcal{L}_{\mathrm{smooth}}$, and  $\breve{\bm{\theta}}^*$ denotes the corresponding optimizer,  it can be seen that for $t \in (t_{m-1},t_m]$,
$ \widetilde h_{\theta_m^*}(\kappa, x, t) = \breve h_{\breve\theta_m^*}(\kappa, x, t)+c_m,  $ where $c_m =\sum_{k=m}^{M_0} \ln C_k$. 
Hence,  the learned score field is invariant:
\[
\widetilde s(\kappa,x,t) \equiv \nabla_x \widetilde h_{\theta_m^*}(\kappa,x,t)
=
\nabla_x \breve h_{\breve{\theta}_m^*}(\kappa,x,t),
\quad t \in (t_{m-1},t_m].
\]
Our approach leads to the following algorithm.
\begin{algorithm}[H]
\caption{Approximation of backward conditional score function and smoothing distribution}
\label{alg:pinn-backward-filter}
\begin{enumerate}
    \item \textbf{Input:} Parameter $\kappa$, observations $\by_{1:M_0}$, observation times $\{t_m\}_{m=1}^{M_0}$.

    \item Initialize the interval-wise neural networks
    $\{\widetilde h_{\theta_m}\}_{m=1}^{M_0}$.

    \item Train the networks by solving
  $
    \bm{\theta}^*
    =
    \arg\min_{\bm{\theta}}
    \mathcal{L}_{\mathrm{smooth}}(\kappa, \bm{\theta})
  $

    \item Compute the learned conditional score function $\widetilde s_{\bm{\theta}^*}(\kappa,x,t)$ by \eqref{eq:learned-score}.
   
    \item \textbf{Output:} Trained networks $\{\widetilde h_{\theta_m^*}\}_{m=1}^{M_0}$ and score function $\widetilde s_{\bm{\theta}^*}$.
\end{enumerate}
\end{algorithm}
This training strategy is related in spirit to physics-informed neural networks (PINNs) \cite{RaPeKa19, ChYaDuKa21}, in that it enforces a governing PDE through a residual loss. However, unlike standard PINN formulations that target continuous solutions of classical PDEs, our setting involves a discontinuous target function $h$ and a sequence of distinct but coupled NNs $\{\widetilde h_{\theta_m}\}$ across successive time intervals, which must simultaneously satisfy the PDE within each interval and the observation-driven jump conditions at the observation times.

\subsection{E-Step (Second half) and M-step: Computing and Maximizing ELBO}

Suppose $\kappa_i$ is the SDE parameter at epoch $i$ of the variational EM algorithm, and let $\bm{\theta}^{(i),*}(\kappa_i) \equiv \bm{\theta}^{(i),*} = (\theta^{(i),*}_{1}, \dots, \theta^{(i),*}_{M_0})$ denote the parameters obtained from the first half of $E$-step of next epoch, $i+1$. Then the law  $\widetilde Q^{(i),*}\equiv \widetilde Q_{\bm{\theta}^{(i), *}}$ of the SDE
\[
d\widetilde X(t)
=
\widetilde b_{\bm{\theta}^{(i),*}}(\kappa,\widetilde X(t),t)\,dt
+
\sigma(\widetilde X(t))\,dW(t), \quad \widetilde b_{\bm{\theta}^{(i),*}}(\kappa,t,x)
=
b(\kappa,x) + a(x)\widetilde s_{\bm{\theta}^{(i),*}}(\kappa,x,t)
\]
approximates $\Post^{(\kappa_i)}(\cdot \mid \by_{1:M_0})$. Here the conditional score $\widetilde s_{\bm{\theta}^{(i),*}}$ is as in \eqref{eq:learned-score}. The  structure of  the posterior SDE  enables us to calculate the ELBO in explicit form by Girsanov's theorem (see \eqref{eq:elbo-expression})
% which can be easily appoximated by Monte-Carlo (MC). 
% Under standard conditions, Girsanov's theorem gives (see \eqref{eq:kl-sde})
% \begin{align}\label{eq:kl-girsanov}
% \mathrm{KL}\big(\widetilde Q^{(i),*} \,\|\, \Pre^{(\kappa)}\big)
% =
% \frac{1}{2}\,\mathbb{E}_{\widetilde Q^{(i),*}}\!\left[
% \int_0^T \widetilde s_{\bm{\theta}^{(i),*}}(t,\widetilde X(t))^\top a(\widetilde X(t))\widetilde s_{\bm{\theta}^{(i),*}}(t,\widetilde X(t))\,dt
% \right].
% \end{align}
% Substituting \eqref{eq:kl-girsanov} into \eqref{eq:elbo-clean} leads to the tractable objective
\begin{align}\label{eq:elbo-controlled}
\begin{aligned}
\mathrm{ELBO} &(\widetilde Q^{(i),*}, \kappa \mid \by_{1:M_0})
=
\mathbb{E}_{\widetilde Q^{(i),*}} \bigg[
\sum_{m=1}^{M_0} \ln \obsden(y_m \mid \widetilde X(t_m))\\
& -
\frac{1}{2} \int_0^T \widetilde s_{\bm{\theta}^{(i),*}}(t,\widetilde X(t))^\top a(\widetilde X(t)) \widetilde s_{\bm{\theta}^{(i),*}}(t,\widetilde X(t))\,dt
\bigg].
\end{aligned}
\end{align}

{\bf Time discretization, reparameterization and MC Approximation:}
Fix a grid $0=s_0<\cdots<s_{N_0}=T$ with step size $\Delta \equiv s_i -s_{i-1} \ll 1$, chosen such that $\{t_m\}\subset\{s_n\}$ and $t_m=s_{n_m}$. For simplicity, assume the initial condition $x_0$ is known and fixed. Also, for notational convenience, we drop the epoch index $i$ and write $\bm{\theta}^*$ instead of $\bm{\theta}^{(i),*}$.
Let $\phi_\kappa$ be the one-step map given by
\[
\phi_\kappa(s,u,y)
=
u + \widetilde b_{\bm{\theta}^*}(\kappa,u,s)\Delta + \sigma(u)\sqrt{\Delta}\,y,
\quad s\in(t_{m-1},t_m].
\]
Then $\widetilde X^D\equiv (x_0,\widetilde X^D_1,\dots,\widetilde X^D_{N_0})$ generated recursively via
\begin{align}
    \label{eq:one-step-map}
    \widetilde X^D_{n+1} = \phi_\kappa(s_n,\widetilde X^D_n,\xi_{n+1}), \quad \xi_n \stackrel{\mrm{iid}}{\sim}\mathcal{N}(0,I_d)
\end{align}
defines  a deterministic map
$
\widetilde X^D
\equiv 
(x_0,\widetilde X^D_1,\dots,\widetilde X^D_{N_0}) \dfeq \Phi(\kappa,x_0,\bm{\xi}_{1:N_0}),
$
where $\Phi$ is the time-ordered composition of $\phi_\kappa$ along the grid $\{s_i\}$ driven by the input noise sequence $\bm{\xi}_{1:N_0}$. This leads to the following {\em Monte-Carlo approximation of ELBO.}
\begin{align}\label{eq:elbo-mc}
\begin{aligned}
\mathrm{ELBO}&_{\mathrm{MC}}(\widetilde Q^{*}, \kappa)
= \ 
\frac{1}{L}\sum_{l=1}^L \Bigg[
\sum_{m=1}^{M_0}
\ln \rho\!\Big(y_m \mid \widetilde X^{D,(l)}_{n_m}\Big) \\
&-\frac{1}{2}\sum_{m=1}^{M_0}\sum_{n: t_{m-1}<s_n\le t_m}
g_{\theta_m^*}^\top(\widetilde X^{D,(l)}_n,s_n)\,
a(\widetilde X^{D,(l)}_n)\,
g_{\theta_m^*}(\widetilde X^{D,(l)}_n,s_n)\,\Delta
\Bigg],
\end{aligned}
\end{align}
where the $\bm{\xi}^{(l)}_{1:N_0}$, $l=1,\dots,L$, are i.i.d. samples and $\widetilde X^{D,(l)} := \Phi(\kappa,x_0,\bm{\xi}^{(l)}_{1:N_0})$. Since $\bm{\xi}^{(l)}_{1:N_0}$ is independent of $\kappa$ and $\Phi$ is a deterministic function of $(\kappa,x_0,\bm{\xi})$, gradients $\nabla_\kappa \mathrm{ELBO}_{\mathrm{MC}}(\widetilde Q^{*},\kappa)$ in the M-step can be computed via backpropagation through the recursion defining $\Phi$.

\begin{algorithm}[H]
\caption{Variational EM algorithm for parameter inference}
\label{alg:pinn-elbo-em}
\begin{enumerate}
    \item \textbf{Input:} Initial parameter $\kappa^{(0)}$, data $\by_{1:M_0}$, initial state $x_0$, grid $\{s_n\}_{n=0}^{N_0}$, number of trajectories $L$, and a convex coefficient $\alpha$.

    \item \textbf{For} $i=0,1,2,\dots$:
    \begin{enumerate}
        \item Run Algorithm~\ref{alg:pinn-backward-filter} with $\kappa=\kappa^{(i)}$ to obtain $\widetilde s_{\bm{\theta}^{(i),*}}$.

        \item Compute the drift function of the posterior SDE:
        $
       \widetilde b_{\bm{\theta}^{(i),*}}(\kappa,t,x) =
      b(\kappa,x) + a(x)\widetilde s_{\bm{\theta}^{(i),*}}(\kappa,x,t)
       $

        \item Simulate $L$ posterior trajectories $\widetilde X^{D,\kappa, (l)} = \Phi\lf(\kappa,x_0,\bm{\xi}^{(l)}_{1:N_0}\ri)$ and compute ELBO by \eqref{eq:elbo-mc}.

        \item Update
        $
        \kappa^{(i+1)}
        =
        \arg\max_{\kappa\in K_0}
        \mathrm{ELBO}_{\mathrm{MC}}(\widetilde Q^{(i),*}, \kappa),
      $
        where $\mathrm{ELBO}_{\mathrm{MC}}$ is as in \eqref{eq:elbo-mc}. The maximization is performed in closed form when available, or by gradient ascent/Adam.
    \end{enumerate}

    \item \textbf{Output:} $\hat\kappa$.
\end{enumerate}
\end{algorithm}

\paragraph{Discussion of existing methodology and its limitations:}
As mentioned in the introduction, a standard approach to estimation of smoothing distribution and SDE-parameter proceeds by time discretization followed by MCMC sampling. Specifically, one approximates the latent path $X_{[0,T]}$ by a high-dimensional vector $\BX_{0:N} = (X(s_0), X(s_1), \ldots, X(s_N))$ over a fine grid $\{s_i\}_{i=0}^N$ with step size $\Delta = s_i - s_{i-1} \ll 1$. Using the Euler--Maruyama scheme,
\[
X(s_i) = X(s_{i-1}) + b(\kappa, X(s_{i-1}))\,\Delta + \sigma(X(s_{i-1})) \sqrt{\Delta}\,\xi_i,
\quad \xi_i \stackrel{\mathrm{i.i.d.}}{\sim} \mathcal{N}(0, I),
\]
one obtains the approximate posterior density
\begin{align*}
\pi_{\mathrm{post}}(\bx_{0:N} \mid \by_{1:M_0})
&\propto 
\prod_{i=1}^N 
\mathcal{N}\!\big(x_i \,\big|\, x_{i-1} + b(\kappa, x_{i-1})\Delta,\; a(x_{i-1})\Delta \big)
\prod_{m=1}^{M_0} \obsden(y_m \mid x_{i_m}),
\end{align*}
where the grid is chosen such that $s_{i_m} = t_m$ and $a(x) = \sigma(x)\sigma(x)^\top$.

Since $\pi_{\mathrm{post}}(\cdot \mid \by_{1:M_0})$ is known only up to a normalization constant, MCMC methods can be used to generate approximate samples. However, this approach has several well-known limitations. First, discretization converts the problem into an artificial high-dimensional density sampling task in $\mathbb{R}^{Nd}$, even when the underlying state process is low-dimensional. Second, the efficiency of MCMC depends critically on the proposal distribution, and mixing typically deteriorates as $\Delta \to 0$. When observations are sparse in time, the posterior over paths becomes highly constrained by only a few observations. While various proposals for sampling diffusion bridges exist, at least in the noise-free setting (e.g., \cite{DeHu06, BeRoSt08, LiChMy10, WGBS17-2}), they are often ineffective for bridging long time intervals in realistic SDE models. As a result, standard MCMC methods often yield poor approximations of the smoothing distribution in such regimes (demonstrated in our experiments).

A further conceptual limitation lies in how discretization interacts with conditioning. In the standard approach, time discretization is applied at the level of the prior SDE $X$, and the resulting approximation is then used to construct an approximate posterior density. While the Euler--Maruyama scheme provides a controlled approximation of the prior dynamics over small time steps, the smoothing distribution is obtained through Bayes' rule, which introduces a nonlinear dependence on the prior path measure. Consequently, although the posterior approximation error is induced by the prior discretization, its effect on the posterior distribution is indirect and generally difficult to characterize.

In contrast, our approach applies discretization to the {\em posterior SDE} $\widetilde X$ directly. As a result, the numerical scheme directly approximates the posterior (smoothing) distribution and, consequently, the associated ELBO. In particular, the discretization error arises at the level of the target quantity of interest, and therefore, by standard Euler--Maruyama approximation results for SDEs, is expected to decrease as the discretization time step  becomes smaller.

\section{Experiments}\label{experiments}

We evaluate the proposed framework on two stochastic systems of increasing complexity: a 2D biochemical model (Michaelis--Menten) and a 4D nonlinear multi-stable system (ring-coupled double well). Both experiments follow an identical setup to enable direct comparison.

{\bf Models:}
Both systems are governed by SDEs with drift $b(x;\kappa)$ and isotropic diffusion $\sigma(x)=\sigma I$.

\textbf{2D Michaelis--Menten (MM):}
The Michaelis--Menten mechanism models describes the interaction conversion of a substrate (S) to product (P) in presence of enzyme (E) through the reactions:
$$
S + E \;\underset{k_{-1}}{\overset{k_1}{\rightleftharpoons}}\; C
\;\xrightarrow{k_2}\; P + E.
$$
The conservation laws, $X_E+X_C =J$ and $X_S+X_C+X_P = J'$, where $J$ and $J'$ are time-independent constants, allows us to embed the system in $[0, \infty)^2$, by modeling $X=(X_S, X_E)$ via an SDE with drift function: 
\[
b_1(x) = -k_1 x_1 x_2 + k_{-1}(J - x_2), \quad
b_2(x) = -k_1 x_1 x_2 + (k_{-1} + 1.0)(J - x_2).
\]
The parameters $k_1$, $k_{-1}$, and $k_2$ represent reaction rates, and we fix  $J=3.0$ and $\sigma=0.3$.

\textbf{4D Ring-Coupled Double Well:} A high-dimensional Ringed Double-well SDE physically models a spatially coupled network of bistable states driven by thermal fluctuations. We consider a 4D-SDE with with $\sigma=0.4$, and drift
\begin{align*}
b_i(x) = k_i x_i - x_i^3 + 0.5 \sum_{j \in \mathscr{N}(i)} (x_j - x_i), \quad i=1,\dots,4,
\end{align*}
where $\mathscr{N}(i)$ denotes the two neighbors of node $i$ in a ring topology, i.e., $\mathscr{N}(1)= \mathscr{N}(3)=\{4, 2\}$, $\mathscr{N}(2)= \mathscr{N}(4)= \{1, 3\}$.

\textbf {Inference task:}
(MM) $\kappa = (k_1, k_{-1})$; $k_2=1$ is fixed. (4D-DW): $\kappa = (k_1,k_2,k_3)$; $k_4=1$ is fixed.
% \[
% \theta =
% \begin{cases}
% (k_1, k_{-1}) & \text{(MM)},\\
% (k_1,k_2,k_3), \; k_4=1 & \text{(4D)}.
% \end{cases}
% \]

Ground truth parameters:
\[
\mrm{\bf MM:}\  (k_1,k_{-1}) = (1.0,1.5), \qquad \mrm{\bf 4D-DW:}\ 
k_1=k_2=k_3=1.0.
\]

% \paragraph{Data generation.}
% In both cases, trajectories are simulated via Euler--Maruyama and subsampled at 5 time points:
% \[
% y_m = X(t_m) + \varepsilon_m, \quad \varepsilon_m \sim \mathcal{N}(0,0.2^2).
% \]

% \begin{center}
% \begin{tabular}{c|c|c}
%  & Michaelis--Menten & 4D Double Well \\
% \hline
% Dimension & $2$ & $4$ \\
% Governing PDE Dimension & $3$ & $5$ \\
% Time horizon & $T=0.5$ & $T=5$ \\
% Step size & $0.001$ & $0.05$ \\
% Initial state & $(1,1)$ & $(1.5,1.5,1.5,1.5)$ \\
% Observation times & $0.1, 0.2, 0.3, 0.4, 0.5$ & $1.0, 2.0, 3.0, 4.0, 5.0$ \\
% \end{tabular}
% \end{center}
\paragraph{Data generation.}
Trajectories are simulated via the Euler--Maruyama scheme and subsampled at 5 equidistant time points, subject to additive Gaussian noise $y_m = X(t_m) + \varepsilon_m$, where $\varepsilon_m \sim \mathcal{N}(0,0.2^2)$. For the 2D Michaelis--Menten system (governing PDE dimension 3), we simulate up to $T=0.5$ with step size $\Delta t=0.001$ from $X(0)=(1,1)$, extracting observations at $t_m \in \{0.1, \dots, 0.5\}$. For the 4D Double Well (governing PDE dimension 5), we simulate up to $T=5$ with $\Delta t=0.05$ from $X(0)=(1.5, 1.5, 1.5, 1.5)$, extracting observations at $t_m \in \{1.0, \dots, 5.0\}$.

% \paragraph{Results.}
% We report parameter convergence and trajectory consistency for both systems.

\looseness=-1
\textbf{Parameter inference:}
Figures~\ref{fig:mm_proposed} and~\ref{fig:ring-double-well} show that our proposed method converges reliably toward true parameters in both models, even from poor initialization. In contrast, the MCMC baseline exhibits poor mixing in the MM system and unstable behavior with large fluctuations in the 4D system.

\textbf{Trajectory quality.}
Figures~\ref{fig:mm_traj} and~\ref{fig:4d_doublewell_traj} illustrate trajectories sampled from the learned posterior SDEs after convergence as recoveries of the latent states in both cases. In both models, sampled paths remain consistent with noisy observations and capture the underlying dynamics smoothly.

\textbf{Implementation Details \& Computational Cost.} 
To enable efficient computation, we model the approximate log-message functions $\widetilde h_\theta$ using a uniform architecture across all observation intervals: 5 hidden layers of 70 neurons with $\tanh$ activations for smooth gradients. While we employ distinct NNs across diiferent observation intervals, formulating the $\mathcal{L}_{\mrm{smooth}}$ as a global sum allows the gradient computations for the PDE residuals and jump conditions to be batched and distributed across multiple GPUs. This design ensures the framework remains scalable as the number of observations increases. For the specific experiments presented in this work—both featuring exactly 5 sparse observations—the method executes efficiently locally utilizing standard compute backends (e.g., Apple MPS or NVIDIA CUDA). In these instances, the initial PINN training requires approximately 20 minutes when employing a learning rate decay schedule (e.g., annealing from $10^{-2}$ down to $10^{-4}$), whereas subsequent EM iterations take 1 minute or less by warm-starting the networks with parameters optimized in the previous iteration.

% ------------------ FIGURES (UNCHANGED) ------------------

% \begin{figure}[]
%     \centering
%     \begin{subfigure}[c]{0.45\textwidth}
%         \centering
%         \includegraphics[width=\textwidth]{EM_k1_k_minus1_Michaelis Menton.png}
%         \caption{Proposed PINN-based Method}
%     \end{subfigure}
%     \hfill
%     \begin{subfigure}[c]{0.45\textwidth}
%         \centering
%         \includegraphics[width=\textwidth]{Michaelis_Menten.png}
%         \caption{MCMC Baseline}
%     \end{subfigure}
%     \caption{Parameter inference for the stochastic Michaelis--Menten system.}
%     \label{fig:mm_results}
% \end{figure}

% \begin{figure}[t]
%     \centering
%     \includegraphics[width=0.7\textwidth]{Michaelis-Menten_100Paths.png}
%     \caption{Trajectory samples for the Michaelis--Menten system.}
%     \label{fig:mm_traj}
% \end{figure}

\begin{figure}[t]
    \centering
    % --- TOP ROW ---
    \begin{subfigure}[c]{0.45\textwidth}
        \centering
        \includegraphics[width=\textwidth]{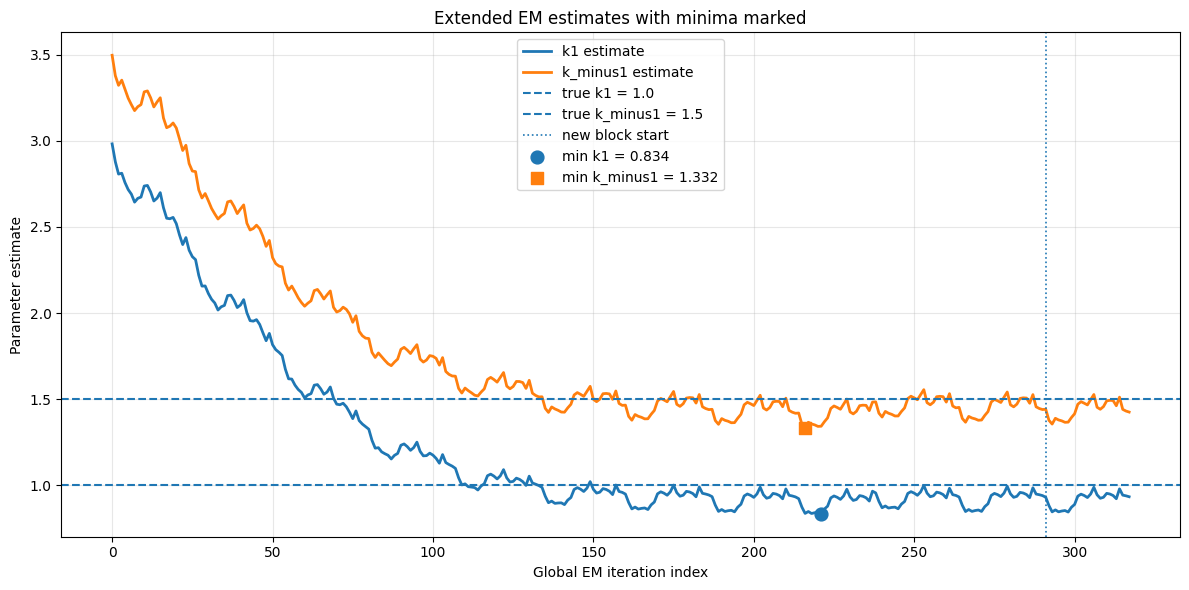}
        \caption{Proposed PINN-based Method}
        \label{fig:mm_proposed}
    \end{subfigure}
    \hfill
    \begin{subfigure}[c]{0.45\textwidth}
        \centering
        \includegraphics[width=\textwidth]{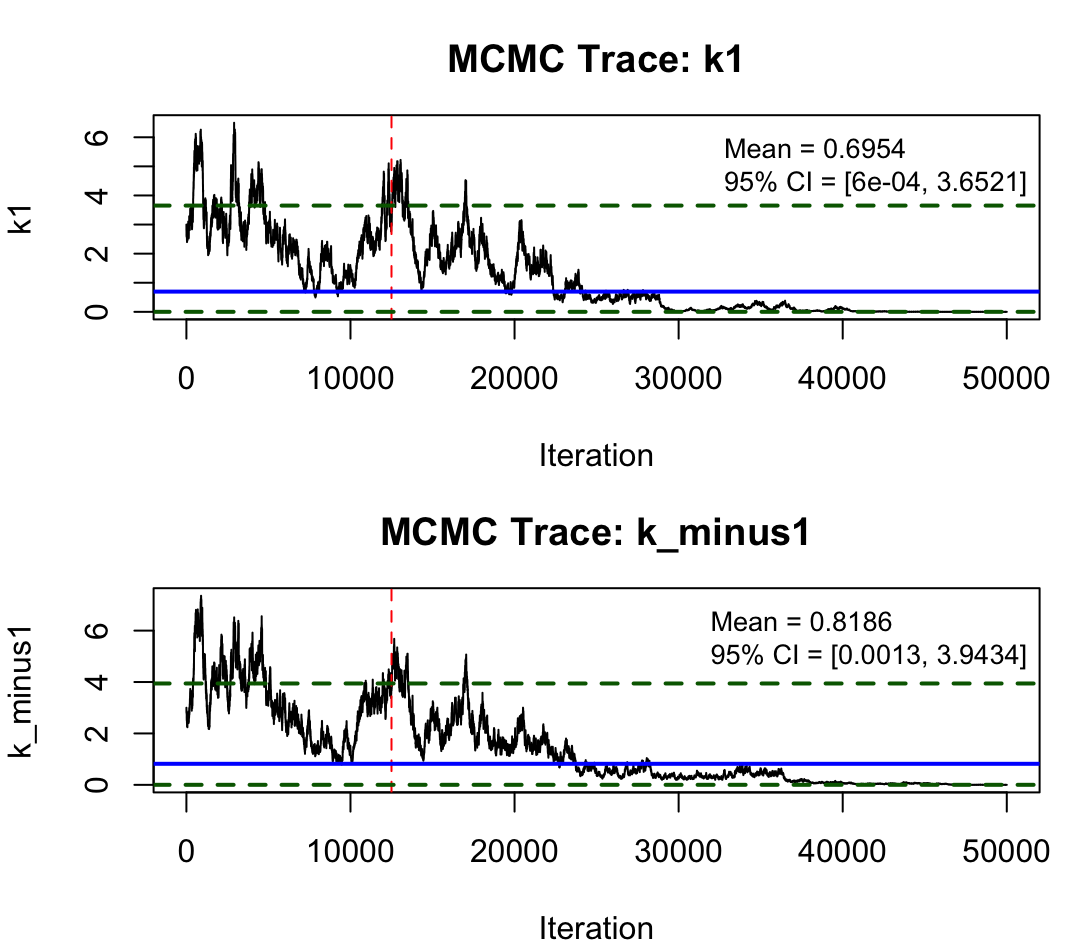}
        \caption{MCMC Baseline}
        \label{fig:mm_mcmc}
    \end{subfigure}
    
    \vspace{4mm} % Adds a small gap between the two rows
    
    % --- BOTTOM ROW (Centered) ---
    \begin{subfigure}[c]{0.70\textwidth} % Slightly reduced from 0.7 to save vertical space
        \centering
        \includegraphics[width=\textwidth]{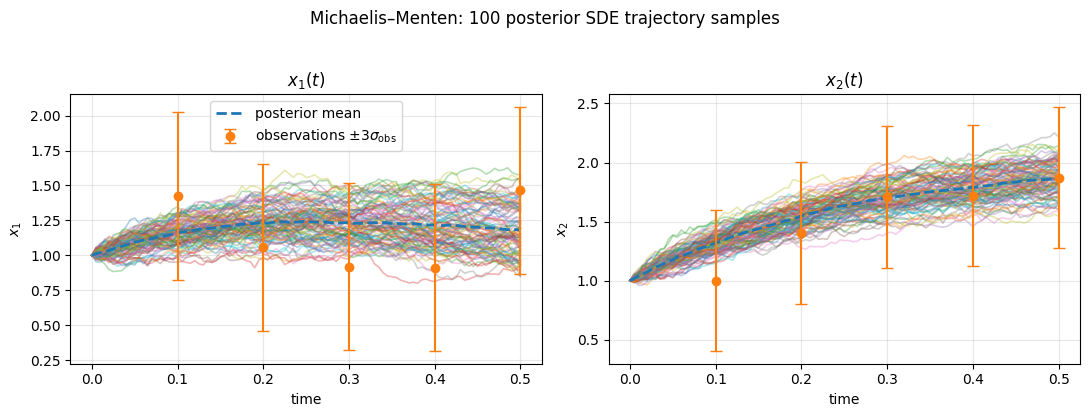}
        \caption{Trajectory samples for the Michaelis--Menten system.}
        \label{fig:mm_traj}
    \end{subfigure}
    
    \vspace{-2mm} % Pulls the main caption a bit closer to the images
    \caption{Parameter inference and trajectory results for the stochastic Michaelis--Menten system.}
    \label{fig:mm_results_combined}
    \vspace{-3mm} % Pulls the subsequent text up to save space
\end{figure}

\begin{figure}[H]
    \centering
    \begin{subfigure}[c]{0.40\textwidth}
        \centering
        \includegraphics[width=\textwidth]{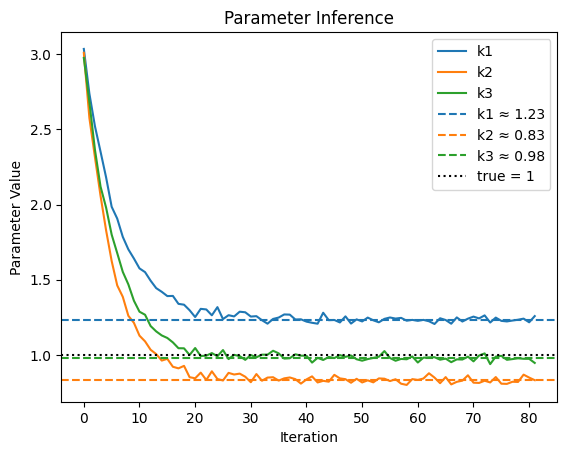}
        \caption{Proposed Method}
    \end{subfigure}
    \hfill
    \begin{subfigure}[c]{0.40\textwidth}
        \centering
        \includegraphics[width=\textwidth]{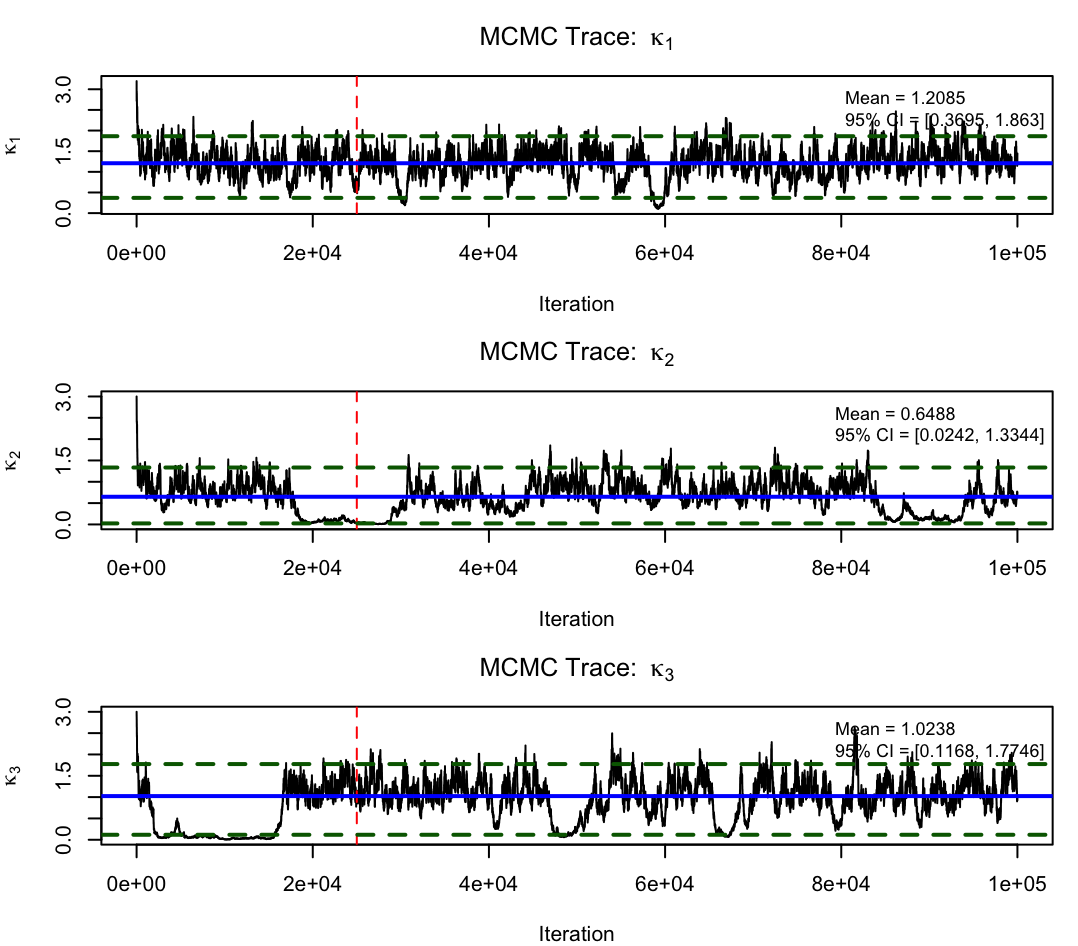}
        \caption{MCMC Baseline}
    \end{subfigure}
    \caption{Parameter inference for the 4D ring-coupled double well system.}
    \label{fig:ring-double-well}
\end{figure}

\begin{figure}[]
    \centering
    \includegraphics[width=0.7\textwidth]{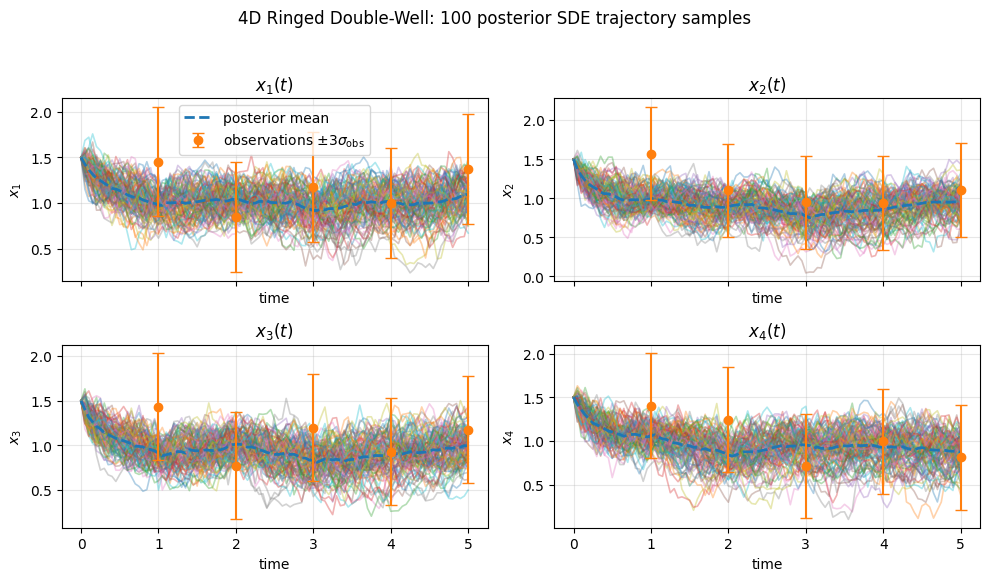}
    \caption{Trajectory samples for the 4D double well system.}
    \label{fig:4d_doublewell_traj}
\end{figure}

\section{Conclusion}\label{conclusion}

We introduced a variational EM framework for joint smoothing and parameter inference in partially observed SDEs. The method is based on a path-space formulation of the smoothing distribution and a learned approximation of the backward-in-time conditional score. This allows the posterior dynamics to be represented through a tractable diffusion process, enabling scalable trajectory sampling and parameter learning from sparse observations.

The proposed approach avoids explicit discretization of prior SDE, which can introduce artificial high-dimensionality,  and instead relies on a neural parameterization of the conditional score combined with a variational ELBO objective. This leads to a unified algorithm that integrates continuous-time stochastic dynamics with VI, and performs favorably in regimes with limited observational data.

Our method depends on the accuracy of the learned conditional score, which is determined by the expressiveness and optimization of the neural parameterization. Inaccuracies in this approximation may affect the induced posterior dynamics and consequently the ELBO optimization. In addition, as with general EM-type procedures, performance can depend on initialization and may be influenced by nonconvexity in the parameter landscape. This can be largely mitigated through multiple initializations.  Developing more expressive architectures and improved optimization strategies for stable score learning in high-dimensional settings is an important direction for future work.
% \bibliographystyle{neurips_2026}
% \bibliography{references}
%\bibliographystyle{unsrt}
\bibliographystyle{plain}
\bibliography{Ref-ML-new}

@article {SuGa16,
    AUTHOR = {Sutter, Tobias and Ganguly, Arnab and Koeppl, Heinz},
     TITLE = {A variational approach to path estimation and parameter
              inference of hidden diffusion processes},
   JOURNAL = {J. Mach. Learn. Res.},
  FJOURNAL = {Journal of Machine Learning Research (JMLR)},
    VOLUME = {17},
      YEAR = {2016},
     PAGES = {Paper No. 190, 37},
      ISSN = {1532-4435},
}

@article {GaMiZh23,
    AUTHOR = {Ganguly, Arnab and Mitra, Riten and Zhou, Jinpu},
     TITLE = {Infinite-dimensional optimization and {B}ayesian nonparametric
              learning of stochastic differential equations},
   JOURNAL = {J. Mach. Learn. Res.},
  FJOURNAL = {Journal of Machine Learning Research (JMLR)},
    VOLUME = {24},
      YEAR = {2023},
     PAGES = {Paper No. [159], 39},
}

@artcle{GaMiZh25,
      title={Nonparametric learning of stochastic differential equations from sparse and noisy data}, 
      author={Arnab Ganguly and Riten Mitra and Jinpu Zhou},
      year={2025},
      eprint={2508.11597},
      archivePrefix={arXiv},
      primaryClass={stat.ML},
      url={https://arxiv.org/abs/2508.11597}, 
}

@book {Bis22a,
    AUTHOR = {Bishwal, Jaya P. N.},
     TITLE = {Parameter estimation in stochastic volatility models},
 PUBLISHER = {Springer, Cham},
      YEAR = {2022},
     PAGES = {xxx+613},
      ISBN = {978-3-031-03860-0; 978-3-031-03861-7},
       DOI = {10.1007/978-3-031-03861-7},
 }

@book {Kut03,
    AUTHOR = {Kutoyants, Yury A.},
     TITLE = {Statistical inference for ergodic diffusion processes},
    SERIES = {Springer Series in Statistics},
 PUBLISHER = {Springer-Verlag London, Ltd., London},
      YEAR = {2004},
     PAGES = {xiv+481},
      ISBN = {1-85233-759-1},
}

@book {Iacus08,
    AUTHOR = {Iacus, Stefano M.},
     TITLE = {Simulation and inference for stochastic differential
              equations},
    SERIES = {Springer Series in Statistics},
      NOTE = {With R examples},
 PUBLISHER = {Springer, New York},
      YEAR = {2008},
     PAGES = {xviii+284},
      ISBN = {978-0-387-75838-1},
}

@article {OpArch09,
    AUTHOR = {Opper, Manfred and Archambeau, C\'{e}dric},
     TITLE = {The variational {G}aussian approximation revisited},
   JOURNAL = {Neural Comput.},
  FJOURNAL = {Neural Computation},
    VOLUME = {21},
      YEAR = {2009},
    NUMBER = {3},
     PAGES = {786--792},
      ISSN = {0899-7667},
}

@incollection {ArOp11,
    AUTHOR = {Archambeau, C\'{e}dric and Opper, Manfred},
     TITLE = {Approximate inference for continuous-time {M}arkov processes},
 BOOKTITLE = {Bayesian time series models},
     PAGES = {125--140},
 PUBLISHER = {Cambridge Univ. Press, Cambridge},
      YEAR = {2011},
}

@incollection{CsOp13,
title = {Approximate inference in latent Gaussian-Markov models from continuous time observations},
author = {Cseke, Botond and Opper, Manfred and Sanguinetti, Guido},
booktitle = {Advances in Neural Information Processing Systems 26},
editor = {C. J. C. Burges and L. Bottou and M. Welling and Z. Ghahramani and K. Q. Weinberger},
pages = {971--979},
year = {2013},
publisher = {Curran Associates, Inc.},
}

@article{GoWi11,
author = {Andrew Golightly  and Darren J. Wilkinson },
title = {Bayesian parameter inference for stochastic biochemical network models using particle Markov chain Monte Carlo},
journal = {Interface Focus},
volume = {1},
number = {6},
pages = {807-820},
year = {2011},
}

@article {GoWi05,
    AUTHOR = {Golightly, A. and Wilkinson, D. J.},
     TITLE = {Bayesian inference for stochastic kinetic models using a
              diffusion approximation},
   JOURNAL = {Biometrics},
  FJOURNAL = {Biometrics. Journal of the International Biometric Society},
    VOLUME = {61},
      YEAR = {2005},
    NUMBER = {3},
     PAGES = {781--788},
      ISSN = {0006-341X},
}

@article {GoWi08,
    AUTHOR = {Golightly, A. and Wilkinson, D. J.},
     TITLE = {Bayesian inference for nonlinear multivariate diffusion models
              observed with error},
   JOURNAL = {Comput. Statist. Data Anal.},
  FJOURNAL = {Computational Statistics \& Data Analysis},
    VOLUME = {52},
      YEAR = {2008},
    NUMBER = {3},
     PAGES = {1674--1693},
      ISSN = {0167-9473},
}

@article {WGBS17,
    AUTHOR = {Whitaker, Gavin A. and Golightly, Andrew and Boys, Richard J.
              and Sherlock, Chris},
     TITLE = {Bayesian inference for diffusion-driven mixed-effects models},
   JOURNAL = {Bayesian Anal.},
  FJOURNAL = {Bayesian Analysis},
    VOLUME = {12},
      YEAR = {2017},
    NUMBER = {2},
     PAGES = {435--463},
      ISSN = {1936-0975},
   MRCLASS = {62F15 (60H10)},
  MRNUMBER = {3620740},
       DOI = {10.1214/16-BA1009},
       URL = {https://doi-org.libezp.lib.lsu.edu/10.1214/16-BA1009},
}

@article {Chib01,
    AUTHOR = {Elerian, Ola and Chib, Siddhartha and Shephard, Neil},
     TITLE = {Likelihood inference for discretely observed nonlinear diffusions},
   JOURNAL = {Econometrica},
  FJOURNAL = {Econometrica. Journal of the Econometric Society},
    VOLUME = {69},
      YEAR = {2001},
    NUMBER = {4},
     PAGES = {959--993},
      ISSN = {0012-9682},
}

@article {RoSt01,
    AUTHOR = {Roberts, G. O. and Stramer, O.},
     TITLE = {On inference for partially observed nonlinear diffusion models
              using the {M}etropolis-{H}astings algorithm},
   JOURNAL = {Biometrika},
  FJOURNAL = {Biometrika},
    VOLUME = {88},
      YEAR = {2001},
    NUMBER = {3},
     PAGES = {603--621},
      ISSN = {0006-3444},
}

@article {BPRF06,
    AUTHOR = {Beskos, Alexandros and Papaspiliopoulos, Omiros and Roberts,
              Gareth O. and Fearnhead, Paul},
     TITLE = {Exact and computationally efficient likelihood-based
              estimation for discretely observed diffusion processes},
      NOTE = {With discussions and a reply by the authors},
   JOURNAL = {J. R. Stat. Soc. Ser. B Stat. Methodol.},
  FJOURNAL = {Journal of the Royal Statistical Society. Series B.
              Statistical Methodology},
    VOLUME = {68},
      YEAR = {2006},
    NUMBER = {3},
     PAGES = {333--382},
      ISSN = {1369-7412},
}

@article {Li13,
    AUTHOR = {Li, Chenxu},
     TITLE = {Maximum-likelihood estimation for diffusion processes via
              closed-form density expansions},
   JOURNAL = {Ann. Statist.},
  FJOURNAL = {The Annals of Statistics},
    VOLUME = {41},
      YEAR = {2013},
    NUMBER = {3},
     PAGES = {1350--1380},
      ISSN = {0090-5364},
}

@article {ChCh11,
    AUTHOR = {Chang, Jinyuan and Chen, Song Xi},
     TITLE = {On the approximate maximum likelihood estimation for diffusion
              processes},
   JOURNAL = {Ann. Statist.},
  FJOURNAL = {The Annals of Statistics},
    VOLUME = {39},
      YEAR = {2011},
    NUMBER = {6},
     PAGES = {2820--2851},
      ISSN = {0090-5364,2168-8966},
       DOI = {10.1214/11-AOS922},
}

@article {Saha02,
    AUTHOR = {A\"{\i}t-Sahalia, Yacine},
     TITLE = {Maximum likelihood estimation of discretely sampled
              diffusions: a closed-form approximation approach},
   JOURNAL = {Econometrica},
  FJOURNAL = {Econometrica. Journal of the Econometric Society},
    VOLUME = {70},
      YEAR = {2002},
    NUMBER = {1},
     PAGES = {223--262},
      ISSN = {0012-9682},
}

@article {Saha08,
    AUTHOR = {A\"it-Sahalia, Yacine},
     TITLE = {Closed-form likelihood expansions for multivariate diffusions},
   JOURNAL = {Ann. Statist.},
  FJOURNAL = {The Annals of Statistics},
    VOLUME = {36},
      YEAR = {2008},
    NUMBER = {2},
     PAGES = {906--937},
      ISSN = {0090-5364,2168-8966},
      DOI = {10.1214/009053607000000622},
}

@article {Yos92,
    AUTHOR = {Yoshida, Nakahiro},
     TITLE = {Estimation for diffusion processes from discrete observation},
   JOURNAL = {J. Multivariate Anal.},
  FJOURNAL = {Journal of Multivariate Analysis},
    VOLUME = {41},
      YEAR = {1992},
    NUMBER = {2},
     PAGES = {220--242},
      ISSN = {0047-259X,1095-7243},
      DOI = {10.1016/0047-259X(92)90068-Q},
}

@article {Kes97,
    AUTHOR = {Kessler, Mathieu},
     TITLE = {Estimation of an ergodic diffusion from discrete observations},
   JOURNAL = {Scand. J. Statist.},
  FJOURNAL = {Scandinavian Journal of Statistics. Theory and Applications},
    VOLUME = {24},
      YEAR = {1997},
    NUMBER = {2},
     PAGES = {211--229},
      ISSN = {0303-6898,1467-9469},
       DOI = {10.1111/1467-9469.00059},
}

@article {BeRoSt08,
    AUTHOR = {Beskos, Alexandros and Roberts, Gareth and Stuart, Andrew and
              Voss, Jochen},
     TITLE = {M{CMC} methods for diffusion bridges},
   JOURNAL = {Stoch. Dyn.},
  FJOURNAL = {Stochastics and Dynamics},
    VOLUME = {8},
      YEAR = {2008},
    NUMBER = {3},
     PAGES = {319--350},
      ISSN = {0219-4937,1793-6799},
}

@article {FePaRo08,
    AUTHOR = {Fearnhead, Paul and Papaspiliopoulos, Omiros and Roberts,
              Gareth O.},
     TITLE = {Particle filters for partially observed diffusions},
   JOURNAL = {J. R. Stat. Soc. Ser. B Stat. Methodol.},
  FJOURNAL = {Journal of the Royal Statistical Society. Series B.
              Statistical Methodology},
    VOLUME = {70},
      YEAR = {2008},
    NUMBER = {4},
     PAGES = {755--777},
         ISSN = {1369-7412},
}

@article {DeHu06,
    AUTHOR = {Delyon, Bernard and Hu, Ying},
     TITLE = {Simulation of conditioned diffusion and application to
              parameter estimation},
   JOURNAL = {Stochastic Process. Appl.},
  FJOURNAL = {Stochastic Processes and their Applications},
    VOLUME = {116},
      YEAR = {2006},
    NUMBER = {11},
     PAGES = {1660--1675},
      ISSN = {0304-4149,1879-209X},
}

@article {BlSo14,
    AUTHOR = {Bladt, Mogens and S\o rensen, Michael},
     TITLE = {Simple simulation of diffusion bridges with application to
              likelihood inference for diffusions},
   JOURNAL = {Bernoulli},
  FJOURNAL = {Bernoulli. Official Journal of the Bernoulli Society for
              Mathematical Statistics and Probability},
    VOLUME = {20},
      YEAR = {2014},
    NUMBER = {2},
     PAGES = {645--675},
      ISSN = {1350-7265,1573-9759},
}

@article {LiChMy10,
    AUTHOR = {Lin, Ming and Chen, Rong and Mykland, Per},
     TITLE = {On generating {M}onte {C}arlo samples of continuous diffusion
              bridges},
   JOURNAL = {J. Amer. Statist. Assoc.},
  FJOURNAL = {Journal of the American Statistical Association},
    VOLUME = {105},
      YEAR = {2010},
    NUMBER = {490},
     PAGES = {820--838},
      ISSN = {0162-1459,1537-274X},
}

@article {WGBS17-2,
    AUTHOR = {Whitaker, Gavin A. and Golightly, Andrew and Boys, Richard J.
              and Sherlock, Chris},
     TITLE = {Improved bridge constructs for stochastic differential
              equations},
   JOURNAL = {Stat. Comput.},
  FJOURNAL = {Statistics and Computing},
    VOLUME = {27},
      YEAR = {2017},
    NUMBER = {4},
     PAGES = {885--900},
      ISSN = {0960-3174,1573-1375},
}

@article{MPS25,
    AUTHOR = {Margossian, Charles C. and Pillaud-Vivien, Loucas and Saul, Lawrence K.},
     TITLE = {Variational inference for uncertainty quantification: an analysis of trade-offs},
   JOURNAL = {J. Mach. Learn. Res.},
  FJOURNAL = {Journal of Machine Learning Research (JMLR)},
    VOLUME = {26},
      YEAR = {2025},
     PAGES = {1--41},
}

@inproceedings{MMYG23,
    AUTHOR = {Modi, Chirag and Margossian, Charles C. and Yao, Yuling and Gower, Robert M. and Blei, David M. and Saul, Lawrence K.},
     TITLE = {Variational inference with {G}aussian score matching},
 BOOKTITLE = {Advances in Neural Information Processing Systems (NeurIPS) 36},
      YEAR = {2023},
     PAGES = {29935--29950},
}

@article{CaTaGa24,
  title={A survey on generative diffusion models},
  author={Cao, Hanqun and Tan, Cheng and Gao, Zhangyang and Xu, Yilun and Chen, Guangyong and Heng, Pheng-Ann and Li, Stan Z},
  journal={IEEE transactions on knowledge and data engineering},
  volume={36},
  number={7},
  pages={2814--2830},
  year={2024},
  publisher={IEEE}
}

@article{TzRa19-1,
  title={Neural stochastic differential equations: Deep latent gaussian models in the diffusion limit},
  author={Tzen, Belinda and Raginsky, Maxim},
  journal={arXiv preprint arXiv:1905.09883},
  year={2019}
}

@inproceedings{SoSoKi21,
  title={Score-Based Generative Modeling through Stochastic Differential Equations},
  author={Song, Yang and Sohl-Dickstein, Jascha and Kingma, Diederik P and Kumar, Abhishek and Ermon, Stefano and Poole, Ben},
  booktitle={International Conference on Learning Representations},
    year={2021},
}

@inproceedings{TzRa19-2,
  title={Theoretical guarantees for sampling and inference in generative models with latent diffusions},
  author={Tzen, Belinda and Raginsky, Maxim},
  booktitle={Conference on Learning Theory},
  pages={3084--3114},
  year={2019},
  organization={PMLR}
}

@article{RaPeKa19,
  title={Physics-informed neural networks: A deep learning framework for solving forward and inverse problems involving nonlinear partial differential equations},
  author={Raissi, Maziar and Perdikaris, Paris and Karniadakis, George E},
  journal={Journal of Computational physics},
  volume={378},
  pages={686--707},
  year={2019},
  publisher={Elsevier}
}

@article{ChYaDuKa21,
  title={Solving inverse stochastic problems from discrete particle observations using the fokker--planck equation and physics-informed neural networks},
  author={Chen, Xiaoli and Yang, Liu and Duan, Jinqiao and Karniadakis, George Em},
  journal={SIAM Journal on Scientific Computing},
  volume={43},
  number={3},
  pages={B811--B830},
  year={2021},
  publisher={SIAM}
}

%%%%%%%%%%%%%%%%%%%%%%%%%%%%%%%%%%%%%%%%%%%%%%%%%%%%%%%%%%%%
\newpage
\appendix

\section{Kolmogorov Forward and Backward Equations}
The generator $\cal{A}$ of a diffusion process $X(t)$ is given by
\begin{align}\label{eq:gen-X}
	\mathcal{A} f(x) = \sum_{i=1}^d b_i(x) \frac{\partial f(x)}{\partial x_i} + \frac12 \sum_{i,j=1}^d a_{ij}(x) \frac{\partial^2 f(x)}{\partial x_i \partial x_j}, \quad f \in C^2(\R^d, \R),
\end{align}
with $a(x) \dfeq \sigma(x)\sigma(x)^\top$. For notational convenience, we suppress the SDE parameter $\kappa$. 

	The dual $\mathcal{A}^*$ is given by
\begin{align}\label{eq:gen-dual-X}
	\mathcal{A}^* p(x) = - \sum_{i=1}^d \frac{\partial}{\partial x_i} \big( b_i(x) \, p(x) \big) + \frac12 \sum_{i,j=1}^d \frac{\partial^2}{\partial x_i \partial x_j} \big( a_{ij}(x) \, p(x) \big), \quad p \in C^2(\R^d, \R).
\end{align}

% Here, by a slight abuse of notation, we denoted the density of the probability measure $\initdist$ also by $\initdist$.	
	
{\em Kolmogorov equations:} The {\em Kolmogorov forward equation} or {\em Fokker--Planck equation} describes the evolution of the probability density $p(x,t)$ of $X(t)$, when $X(0) \sim \initdist$
	\[
	\partial_t p(x,t) = \mathcal{A}^* p(x,t) 
	= - \nabla_x \cdot \big( b(x) \, p(x,t) \big) + \frac12 \nabla_x \cdot \nabla_x \cdot \big( a(x) \, p(x,t) \big), \qquad p(x,0) = p_0(x)
	\]
The {\em Kolmogorov backward equation} describes the evolution of the conditional expectation
\[
u(x,t) := \mathbb{E}[f(X_T) \mid X_t = x]:
\]
and is given by
\begin{align}
    \label{eq:kol-back}
   	\partial_t u(x,t) + \mathcal{A} u(x,t) = 0, \quad u(x,T) = f(x). 
\end{align}
which, as the name suggests, needs to be solved backwards in time.	
	
{\em Equation for transition Densities}: Let $p(t,x,t',\cdot)$ denote the transition density of $X(t')$ given $X(t) = x$:
\[
 \mathbb{P}(X_{t'} \in A \mid X_t = x) =\int p(t,x,t',z)\, dz.
\]
Note that
\begin{itemize}
\item
 for any fixed $t,t'$ and $z$, $p(t,\cdot,t',z)$ is not necessarily a probability density function, i.e., $\displaystyle \int p(t,x,t',z)\ dx \neq 1;$

\item if $\initdist$ is the initial probability density of $X(0)$, then the density $p(\cdot,t)$ of $X(t)$ is of course given as
$p(x,t) = \int \initdist(y) p(0,y,t,x)\ dy.$  
\end{itemize}

Since the SDE considered in this note is time homogeneous, that is the driving functions $\drft(\cdot)$ and $\dffun(\cdot)$ do not depend on time variable $t$, $p(t,x,t',z)$ has the form $p(t,x,t',z) = p(t'-t, x, z)$

\np
It is easy to see from previous discussion that for any fixed $(t_0,x_0)$, $p(t_0,x_0, t,x)$ satisfies the Kolmogorov forward equation
\[\partial_t p(t_0,x_0, t,x) = \mathcal{A}^* p(t_0,x_0,\cdot,\cdot)(t,x), \quad p(t_0,x_0, t_0,x) = \delta(x-x_0), \]
and for any fixed $(T,x_T)$, $p(t,x,T,x_T)$ satisfies the Kolmogorov backward equation in $[0,T]$
\[
\partial_t p(t,x,T,x_T) + \mathcal{A}_x p(\cdot,\cdot,T,x_T)(t,x) = 0, \quad p(T,x,T,z) = \delta(x-z).
\]

% KEEP FULL APPENDIX CONTENT

\section{Auxiliary results and proofs}

\paragraph{Girsanov theorem and KL divergence between diffusion laws.}
Consider two Itô diffusions on $[0,T]$ with the same diffusion coefficient $\sigma(x)\in\mathbb{R}^{d\times d}$, which is assumed to be uniformly nondegenerate: 
\begin{align*}
dX_0(t) &= b_0(X_0(t))\,dt + \sigma(X_0(t))\,dW(t), \qquad X(0)\sim \mu_0, \\
dX_1(t) &= b_1(X_1(t))\,dt + \sigma(X(_1t))\,dW(t), \qquad X(0)\sim \mu_1.
\end{align*}
Let $\Pi_0$ and $\Pi_1$ denote the laws of $X_0$ and $X_1$ on $C([0,T];\mathbb{R}^d)$. Assume that $\Pi_0 \ll \Pi_1$, which holds under standard conditions (e.g., Novikov’s condition) when the diffusion coefficients coincide. Define the function
\[
u(x) = \sigma(x)^{-1}\big(b_1(x)-b_0(x)\big).
\]
Then, by Girsanov’s theorem, the Radon--Nikodym derivative of $\Pi_1$ with respect to $\Pi_0$ is given by
\begin{align*}
%\label{eq:girsanov-rn}
\frac{d\Pi_1}{d\Pi_0}(X)
=
\frac{d\mu_1}{d\mu_0}(X(0))
\exp\!\left(
\int_0^T u(X(t))^\top dW(t)
-
\frac{1}{2}\int_0^T \|u(X(t))\|^2 dt
\right),
\end{align*}
where $W(t)$ is the $\Pi_0$-Brownian motion driving the first SDE.

Taking expectation with respect to $\Pi_1$ and using the fact that
\[
\widetilde W(t) = W(t) - \int_0^t u(X(s))\,ds
\]
is a Brownian motion under $\Pi_1$, we obtain
\begin{align}
\mathrm{KL}(\Pi_1 \,\|\, \Pi_0)
&= \mathbb{E}_{\Pi_1}\!\left[\ln \frac{d\Pi_1}{d\Pi_0}\right] = \mathrm{KL}(\mu_1 \,\|\, \mu_0)
+ \frac{1}{2}\,\mathbb{E}_{\Pi_1}\!\left[
\int_0^T \|u(X_1(t))\|^2 dt
\right].
\label{eq:kl-sde}
\end{align}
In particular, when the initial distributions coincide, the KL divergence reduces to
\[
\mathrm{KL}(\Pi_1 \,\|\, \Pi_0)
=
\frac{1}{2}\,\mathbb{E}_{\Pi_1}\!\left[
\int_0^T 
\big(b_1(X_1(t)) - b_0(X_1(t))\big)^\top a^{-1}(X_1(t))\big(b_1(X_1(t)) - b_0(X_1(t))\big)
dt
\right],
\]
where $a=\s\s^\top$.
This representation highlights that finiteness of the KL divergence requires the two diffusions to share the same diffusion coefficient, and expresses the discrepancy between path measures purely in terms of the drift mismatch.

\begin{proof}[Proof of Theorem \ref{thm:elbo-lower-bound}]
    We have to work with KL-divergence of two SDEs in the path space. Notice that
   \begin{align*}
       \mrm{KL}(Q\| \Pre^{(\kappa)} ) =&\ \int_{C([0,T],\R^d)}\lf[ \ln \lf(\f{dQ}{d\Post^{(\kappa)}(\cdot \mid \by_{1:M_0})}\ri) + \ln \lf(\f{\Post^{(\kappa)}(\cdot \mid \by_{1:M_0})}{d \Pre^{(\kappa)}} \ri)\ri] dQ\\
       =&\ \mrm{KL}(Q\|\Post^{(\kappa)}(\cdot \mid \by_{1:M_0}) ) + \mathbb{E}_Q \!\left[
\sum_{m=1}^{M_0} \ln \obsden(y_m \mid X(t_m))
\right] - \ell(\kappa \mid \by_{1:M_0}),
   \end{align*} 
where for the last step we used \eqref{eq:prior-post-RN}. Rearranging shows for any $Q$ 
$$\ell(\kappa \mid \by_{1:M_0})  = \mrm{KL}(Q\|\Post^{(\kappa)}(\cdot \mid \by_{1:M_0}) )+  \mathrm{ELBO}(Q, \kappa \mid \by_{1:M_0}) \geq \mathrm{ELBO}(Q, \kappa \mid \by_{1:M_0}),$$
with the equality holding in the last step when $Q = \Post^{(\kappa)}(\cdot \mid \by_{1:M_0}).$ 
\end{proof}

Therefore, when $\widetilde Q = \mrm{Law}(\widetilde X)$, where $\widetilde X$ is defined in \eqref{eq:SDE-post-approx}, is used as a variational approximation to the smoothing distribution $\Post^{(\kappa)}(\cdot \mid \by_{1:M_0}),$ ELBO is given by
\begin{align}\label{eq:elbo-expression}
\mathrm{ELBO}(\widetilde Q, \kappa \mid \by_{1:M_0})
=
\mathbb{E}_{\widetilde Q} \bigg[
\sum_{m=1}^{M_0} \ln \obsden(y_m \mid \widetilde X(t_m))
-
\frac{1}{2} \int_0^T 
\widetilde s(t,\widetilde X(t))^\top 
a^{-1}(\widetilde X(t))
\widetilde s(t,\widetilde X(t))\,dt
\bigg].
\end{align}

\begin{proof}[Proof of Proposition \ref{prop: log-w-pde}] 
 First observe that 
\begin{align}\label{eq:h-time-deriv}
\partial_t h(\kappa,x,t) = \f{\partial_t w(\kappa, x,t)}{w(\kappa,x,t)} = - \f{ \cal{A}^{(\kappa)} w(\kappa, x,t)}{w(\kappa,x,t)}, \quad t\in(t_m,t_{m+1}].
\end{align}
Next, we have
\begin{align*}
 \cal{A}^{(\kappa)} h(\kappa, x,t) =&\ \f{ \cal{A}^{(\kappa)} w(\kappa, x,t)}{w(\kappa,x,t)} -\f{1}{2}\sum_{i,j}a_{ij}(x) \f{\partial_{x_i} w(\kappa,x,t)\partial_{x_j} w(\kappa,x,t)}{w^2(\kappa,x,t)}\\
=&\ \f{ \cal{A}^{(\kappa)} w(\kappa, x,t)}{w(\kappa,x,t)} - \f{1}{2}\sum_{i,j}a_{ij}(x)\partial_{x_i} h(\kappa, x,t)\partial_{x_j} h(\kappa, x,t)\\
=&\ \f{ \cal{A}^{(\kappa)} w(\kappa, x,t)}{w(\kappa,x,t)} - \f{1}{2}\nabla^T_xh(\kappa, x,t) a(x)\nabla_x h(\kappa, x,t).\end{align*}
It follows from \eqref{eq:h-time-deriv} that  for each $\kappa$, the function $h(\kappa, \cdot,\cdot)$ between observation times $t\in(t_{m-1},t_{m}]$ satisfies the backward PDE
\begin{align*}
% \label{eq:backwardPDE-log-post}
\partial_t h(\kappa, x,t) + \cal{A}^{(\kappa)} h(\kappa, x,t)+\f{1}{2}\nabla^T_xh(\kappa, x,t) a(x)\nabla_x h(\kappa, x,t)=0, \quad a=\s\s^\top
\end{align*}
with the boundary condition at $t_m$ given by
	\begin{equation*}%\label{eq:jump-obs-lnw}
		h(\kappa, x,t_m) = \ln \obsden(y_m\mid x) + h(\kappa, x,t_m^+).
	\end{equation*}
\end{proof}

\section{Numerical Verification against Exact PDE Solutions: 1D GBM \& 1D Double Well}
\label{app:gbm}

To validate the numerical accuracy of the proposed framework, we consider an one-dimensional Geometric Brownian Motion (GBM) and an one-dimensional Double Well equation, for each of which both the Kolmogorov equations and smoothing distributions are well understood. This example allows us to directly compare our proposed approximation with a ground-truth PDE-based solution.

\paragraph{Model.}
Both systems are governed by SDEs with drift $b(x;\theta)$ and isotropic diffusion $\sigma(x)=\sigma I$.

\textbf{1D Geometric Brownian Motion.}
\[
dX(t) = \theta X(t)\,dt + \sigma X(t)\,dW(t), \quad \sigma=0.3.
\]

% \paragraph{Background.}
% The Michaelis--Menten mechanism models the interaction between a substrate and an enzyme through the reactions:
% \[
% S + E \;\underset{k_{-1}}{\overset{k_1}{\rightleftharpoons}}\; C
% \;\xrightarrow{k_2}\; P + E,
% \]
% where $S$ is the substrate, $E$ the enzyme, $C$ the enzyme-substrate complex, and $P$ the product. The parameters $k_1$, $k_{-1}$, and $k_2$ represent reaction rates. Estimating these parameters from noisy observations is a fundamental problem in biochemical systems.

\textbf{1D Double Well System.}
\begin{align*}
dX(t) = \kappa (X(t) - X(t)^3)dt + \sigma dW(t), \quad \sigma=0.5.
\end{align*}

\paragraph{Setup.}
We compare two approaches:
\begin{itemize}
\item \textbf{Finite Difference Method (FDM):} A numerical solver for the Kolmogorov backward equation, which serves as a reference solution.
\item \textbf{Proposed method:} Our neural variational framework, which learns the log-likelihood function and corresponding score.
\end{itemize}

\paragraph{Data generation.}
In both cases, trajectories are simulated via Euler--Maruyama and subsampled at 5 time points:
\[
y_m = X(t_m) + \varepsilon_m, \quad \varepsilon_m \sim \mathcal{N}(0,0.2^2).
\]

\begin{center}
\begin{tabular}{c|c|c}
 & 1D GBM & 1D Double Well \\
\hline
Governing PDE Dimension & $2$ & $2$ \\
Time horizon & $T=0.5$ & $T=5$ \\
Step size & $0.001$ & $0.05$ \\
Initial state & $1.0$ & $3.0$ \\
Observation times & $0.1, 0.2, 0.3, 0.4, 0.5$ & $1.0, 2.0, 3.0, 4.0, 5.0$ \\
\end{tabular}
\end{center}

\paragraph{Comparison of smoothings.}
Figure~\ref{fig:1d_results} compares the function $w(x,t)$ obtained from the PDE solver and the proposed approximation at multiple time points. The solid curves correspond to the FDM solution, while dashed curves represent the neural approximation.

% \begin{figure}[H]
% \centering
% \includegraphics[width=0.5\linewidth]{comparison_result_KBE.png}
% \caption{Comparison of the solution $w(x,t)$ of the Kolmogorov backward equation for GBM. Solid lines correspond to the finite-difference (FDM) solution, while dashed lines correspond to the PINN approximation. The close agreement across time points demonstrates that the neural method accurately captures the underlying PDE solution.}
% \label{fig:gbm_density}
% \end{figure}

\begin{figure}[H]
    \centering
    \begin{subfigure}[c]{0.48\textwidth}
        \centering
        \includegraphics[width=\textwidth]{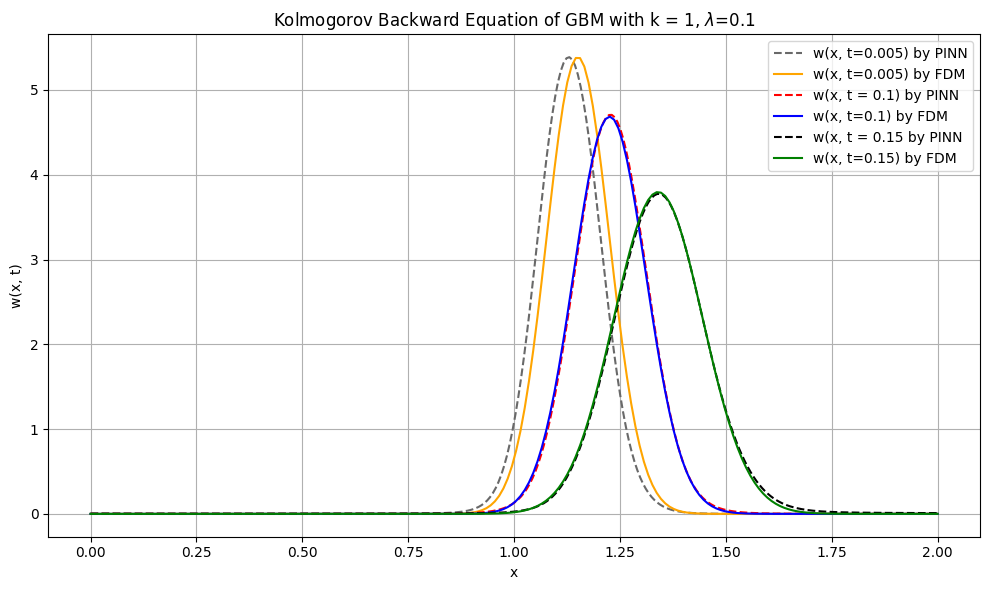}
        \caption{1D GBM}
    \end{subfigure}
    \hfill
    \begin{subfigure}[c]{0.48\textwidth}
        \centering
        \includegraphics[width=\textwidth]{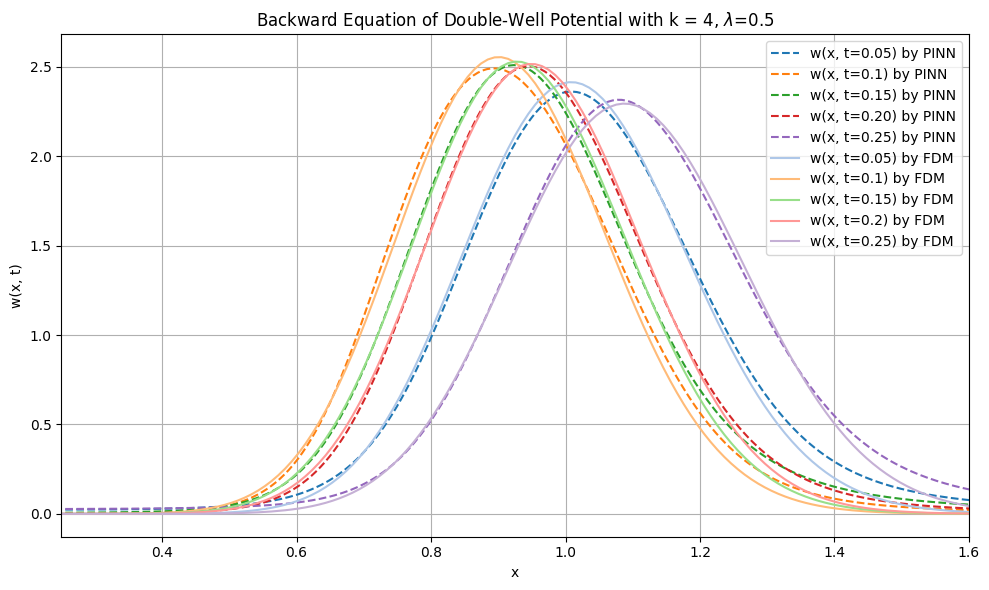}
        \caption{1D Double Well}
    \end{subfigure}
    \caption{Comparison of Smoothings}
    \label{fig:1d_results}
\end{figure}

\paragraph{Parameter inference: Proposed method.}
We next estimate the drift parameter $\kappa$ using the proposed variational method. Figure~\ref{fig:1d_inference} shows the evolution of the estimate over iterations. Starting from an initial value far from the ground truth, the estimate converges rapidly toward $\kappa = 1$.

% \begin{figure}[H]
% \centering
% \includegraphics[width=0.5\linewidth]{drift_estimate_GBM.png}
% \caption{Convergence of the drift parameter estimate $\hat{\kappa}$ using the proposed PINN-based method. The estimate converges rapidly toward the true value $\kappa = 1$.}
% \label{fig:gbm_pinn}
% \end{figure}

\begin{figure}[H]
    \centering
    \begin{subfigure}[c]{0.48\textwidth}
        \centering
        \includegraphics[width=\textwidth]{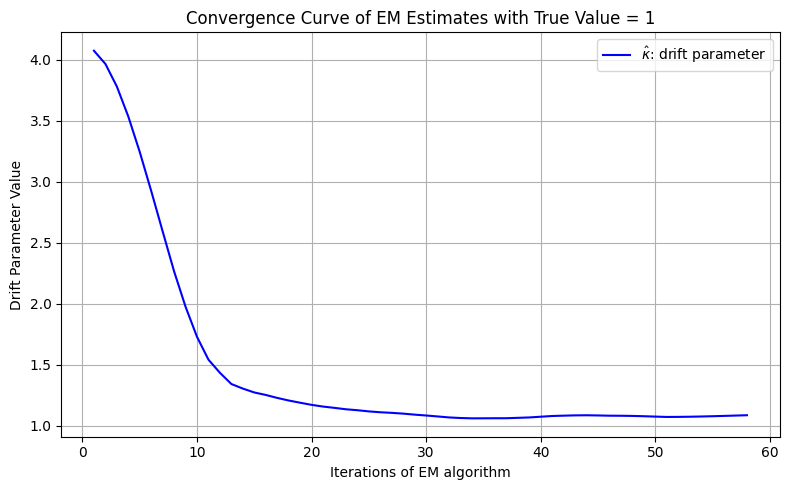}
        \caption{1D GBM with $\theta_{\text{true}}=1.0$}
    \end{subfigure}
    \hfill
    \begin{subfigure}[c]{0.48\textwidth}
        \centering
        \includegraphics[width=\textwidth]{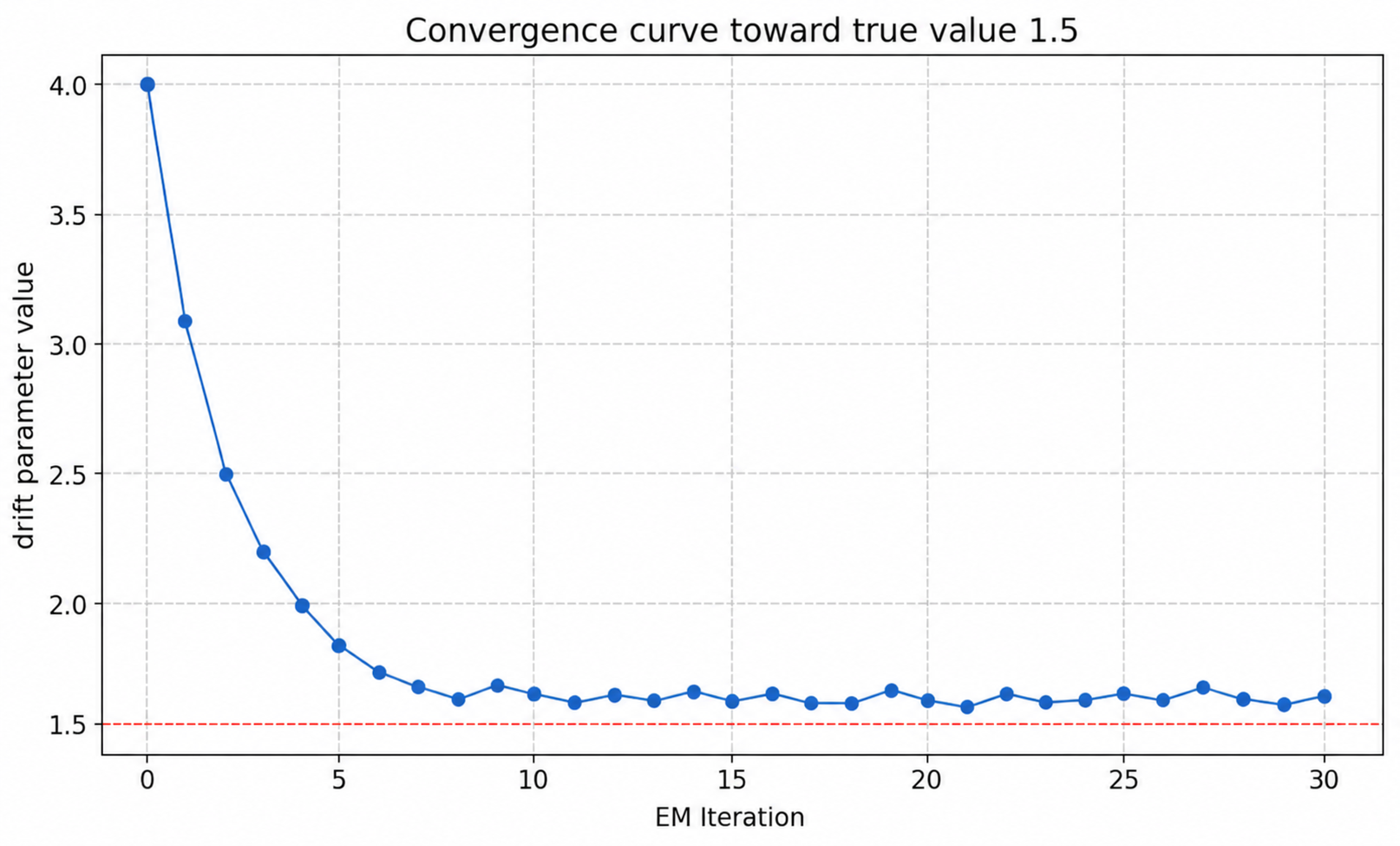}
        \caption{1D double well with $\kappa_{\text{true}}=1.5$}
    \end{subfigure}
    \caption{Parameter inferences via our proposed method.}
    \label{fig:1d_inference}
\end{figure}

\paragraph{Parameter inference: PDE-based EM baseline.}
For comparison, we also perform parameter estimation using a PDE-based EM approach. Figure~\ref{fig:1d_inference_pde} shows the corresponding convergence behavior.

% \begin{figure}[H]
% \centering
% \includegraphics[width=0.5\linewidth]{PDE_approach_EM_GBM.png}
% \caption{Convergence of the parameter estimate using the PDE-based EM approach. The estimate converges to the same value as the proposed method.}
% \label{fig:gbm_pde}
% \end{figure}

\begin{figure}[H]
    \centering
    \begin{subfigure}[c]{0.48\textwidth}
        \centering
        \includegraphics[width=\textwidth]{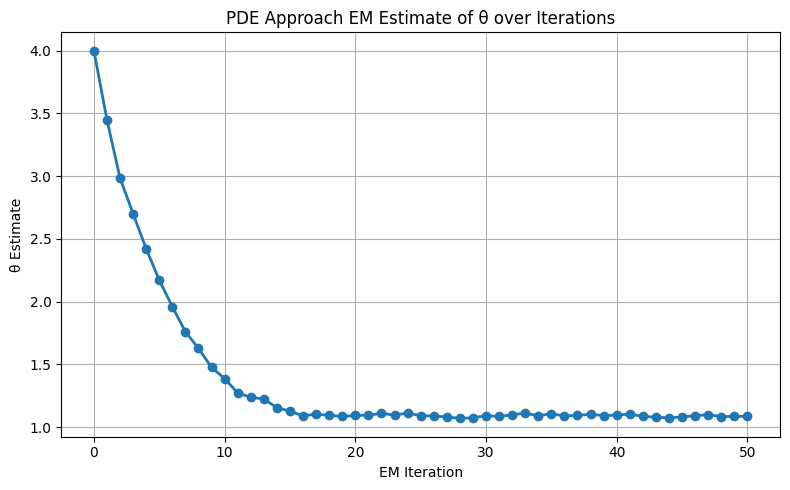}
        \caption{1D GBM with $\theta_{\text{true}}=1.0$}
    \end{subfigure}
    \hfill
    \begin{subfigure}[c]{0.48\textwidth}
        \centering
        \includegraphics[width=\textwidth]{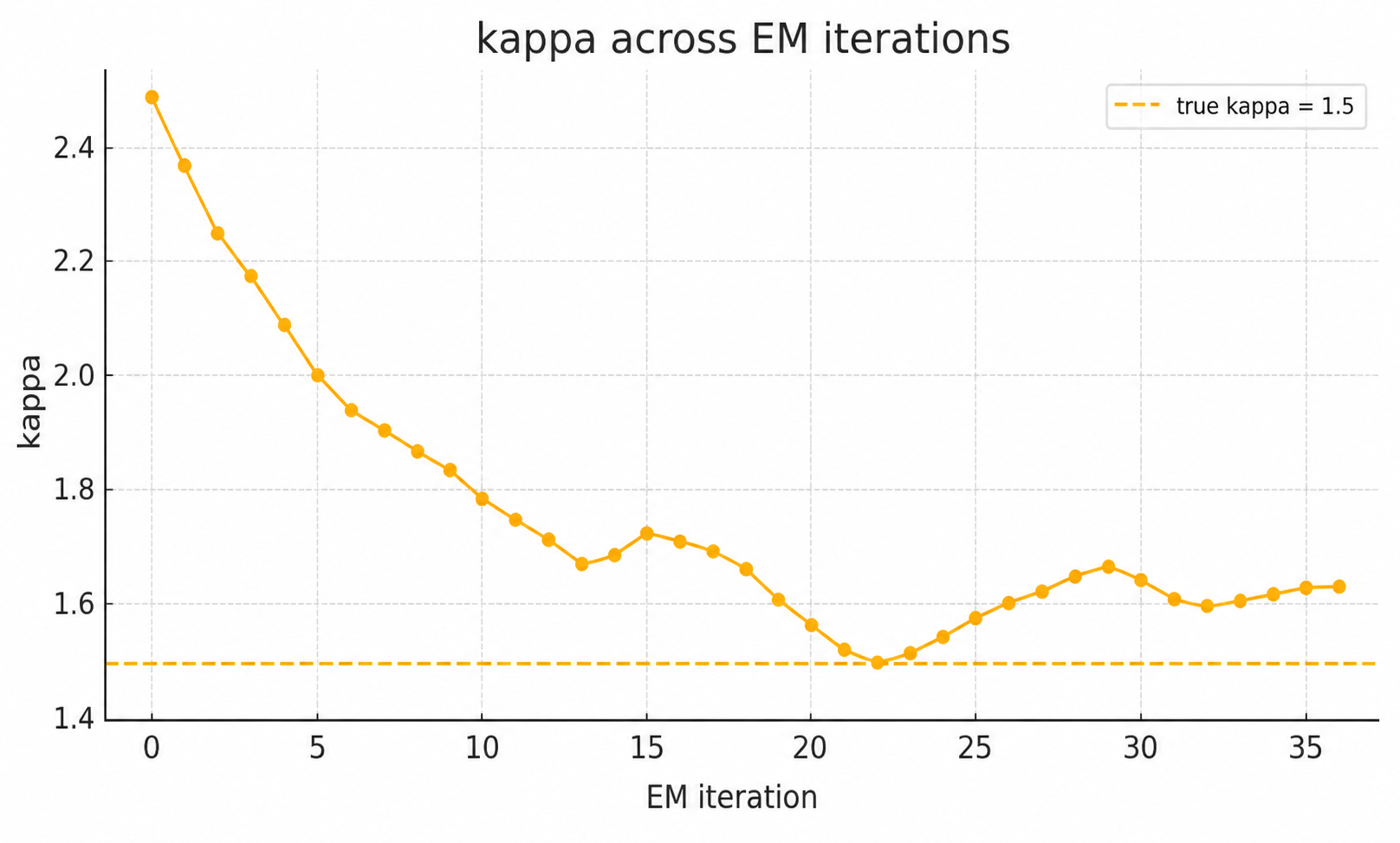}
        \caption{1D double well with $\kappa_{\text{true}}=1.5$}
    \end{subfigure}
    \caption{Parameter inferences via FDM-based method.}
    \label{fig:1d_inference_pde}
\end{figure}

\paragraph{Discussion.}
The results demonstrate excellent agreement between the neural approximation and the PDE-based solution. In particular:
\begin{itemize}
\item The Proposed method accurately reproduces the solution of the Kolmogorov backward equation.
\item The inferred parameter converges to the same value as the PDE-based EM method.
\item The convergence behavior is stable and robust despite nonlinear dependence on the parameter.
\end{itemize}

This example provides strong evidence that the proposed framework is capable of accurately approximating both the smoothing distribution and the underlying parameter, even without explicitly solving the governing PDE.

% \newpage
% \input{checklist.tex}
\end{document}